\RequirePackage[loading]{tracefnt}
\documentclass[twoside]{article}

\usepackage[accepted]{aistats2019}

\usepackage[utf8]{inputenc} 
\usepackage[T1]{fontenc}    
\usepackage{hyperref}       
\usepackage{url}            
\usepackage{booktabs}       
\usepackage{amsfonts}       
\usepackage{nicefrac}       
\usepackage{microtype}      

\usepackage{color}
\usepackage{amssymb}
\usepackage{amsmath}
\usepackage{amsthm}
\usepackage{dsfont}
\usepackage[font={small}]{caption}
\usepackage[export]{adjustbox}
\usepackage{bm}
\usepackage{enumerate}
\usepackage{subcaption}
\usepackage{lipsum}
\usepackage{framed}

\usepackage{todonotes}

\usepackage[numbers]{natbib}

\newcommand{\method}{{\sc Cesp}}
\newcommand{\x}{\mathbf{x}} 
\newcommand{\y}{\mathbf{y}} 
\newcommand{\pmin}{\mathbf{x}} 
\newcommand{\pmax}{\mathbf{y}} 
\newcommand{\z}{\mathbf{z}} 
\renewcommand{\v}{\mathbf{v}} 
\newcommand{\I}{\mathbf{I}} 
\newcommand{\R}{{\mathbb{R}}} 
\newcommand{\K}{{\cal K}} 

\newcommand{\myvec}[2]{\begin{bmatrix}#1 \\ #2 \end{bmatrix}}

\newcommand{\mymatrix}[4]{\begin{bmatrix}#1 & #2 \\ #3 & #4 \end{bmatrix}}

\renewcommand{\H}{\mathrm{H}}
\newcommand{\A}{\mathbf{A}}

\newcommand{\J}{\mathbf{J}}

\renewcommand{\H}{\mathbf{H}} 
\newcommand{\ind}{ \mathds{1}}

\newcommand{\bigo}{{\cal O}}

\newtheorem{theorem}{Theorem}
\newtheorem*{lemma*}{Lemma}
\newtheorem*{corollary*}{Corollary}
\newtheorem*{theorem*}{Theorem}
\newtheorem{lemma}[theorem]{Lemma}

\newtheorem{define}[theorem]{Definition}
\newtheorem{assumption}[theorem]{Assumption}
\newtheorem{innercustomthm}{Lemma}
\newenvironment{lemma_custom_no}[1]
  {\renewcommand\theinnercustomthm{#1}\innercustomthm}
  {\endinnercustomthm}
  
\usepackage{hyperref}
\usepackage{multirow}

\begin{document}

\runningtitle{Local Saddle Point Optimization: A Curvature Exploitation Approach}
\runningauthor{Adolphs, Daneshmand, Lucchi, Hofmann}

\twocolumn[

\aistatstitle{Local Saddle Point Optimization: A Curvature Exploitation Approach}
\aistatsauthor{Leonard Adolphs \And Hadi Daneshmand \And  Aurelien Lucchi \And Thomas Hofmann}
\aistatsaddress{Department of Computer Science, ETH Zurich}

]

\begin{abstract}
    Gradient-based optimization methods are the most popular choice for finding local optima for classical minimization and saddle point problems. Here, we highlight a systemic issue of gradient dynamics that arise for saddle point problems, namely the presence of undesired stable stationary points that are no local optima. We propose a novel optimization approach that exploits curvature information in order to escape from these undesired stationary points.
    We prove that different optimization methods, including gradient method and {\sc Adagrad}, equipped with curvature exploitation can   escape non-optimal stationary points.
    We also provide empirical results on common saddle point problems which confirm the advantage of using curvature exploitation. 
\end{abstract}


\section{INTRODUCTION}
 We consider the problem of finding a \emph{structured}\footnote{Throughout this work, we aim to find saddles that satisfy a particular (local) min-max structure in the input parameters.} saddle point of a smooth objective, namely solving an optimization problem of the form 
\begin{align} \label{eq:saddle-problem}
\min_{\pmin \in \R^k } \max_{\pmax \in \R^d} f(\pmin,\pmax).
\end{align}
Here, we assume that $f$ is smooth in $\pmin$ and $\pmax$ but \textbf{not} necessarily convex in $\pmin$ or concave in $\pmax$. This particular problem arises in many applications, such as generative adversarial networks (GAN)~\cite{goodfellow14gan}, robust optimization~\cite{ben2009robust}, and game theory~\cite{singh2000nash, Leyton-Brown:2008:EGT:1481632}. Solving the saddle point problem in Eq.~\eqref{eq:saddle-problem} is equivalent to finding a point $(\pmin^*,\pmax^*)$ such that 
\begin{align} \label{eq:global_saddle_point}
f(\pmin^*,\pmax) \leq f(\pmin^*, \pmax^*) \leq f(\pmin,\pmax^*) 
\end{align}
holds for all $\pmin \in \R^k$ and $\pmax \in \R^d$. For a non convex-concave function $f$, finding such a saddle point is computationally infeasible. Instead of finding a \emph{global} saddle point for Eq.~\eqref{eq:saddle-problem}, we aim for a more modest goal: finding a \emph{locally} optimal saddle point, i.e. a point $(\pmin^*, \pmax^*)$ for which the condition in Eq.~\eqref{eq:global_saddle_point} holds true in a local neighbourhood around $(\pmin^*,\pmax^*)$.

There is a rich literature on saddle point optimization for the particular class of convex-concave functions, i.e. when $f$ is convex in $\pmin$ and concave in $\pmax$. Although this type of objective function is commonly encountered in applications such as constrained convex minimization, many saddle point problems of interest do not satisfy the convex-concave assumption. Two popular examples that recently emerged in machine learning
are distributionally robust optimization~\cite{gao2016distributionally, sinha2018certifying}, as well as training generative adversarial networks~\cite{goodfellow14gan}. These applications can be framed as saddle point optimization problems which - due to the complex functional representation of the neural networks used as models - do not fulfill the convexity-concavity condition. 

First-order methods are commonly used to solve problem~\eqref{eq:saddle-problem} as they have a cheap per-iteration cost and are therefore easily scalable. One particular method of choice is simultaneous gradient descent/ascent, which performs the following iterative updates,
\begin{align} \label{eq:saddle_gradient}
(\pmin_{t+1}, \pmax_{t+1}) &= (\pmin_{t}, \pmax_{t}) + \eta_t \left(-\nabla_\pmin f_t,\nabla_\pmax f_t \right) \\ f_t :&= f (\pmin_{t}, \pmax_{t}), \nonumber
\end{align}
where $\eta_t > 0$ is a chosen step size which can, e.g., decrease with time $t$ or be a bounded constant (i.e. $\eta_t = \eta$).
The convergence analysis of the above iterate sequence is typically tied to a strong/strict convexity-concavity property of the objective function defining the dynamics. Under such conditions, the gradient method is guaranteed to converge to a desired saddle point~\cite{arrow1958studies}. These conditions can also be relaxed to some extent, which will be further discussed in Section~\ref{sec:related_work}.

It is known that the gradient method is locally asymptotically stable~\cite{mescheder2017numerics}; but stability alone is not sufficient to guarantee convergence to a locally \emph{optimal} saddle point. Through an example, we will later illustrate that the gradient method is indeed stable at some undesired stationary points, at which the \emph{structural} min-max property~\footnote{This property refers to the function being a local minimum in $\pmin$ and a maximum in $\pmax$.} is not met. This is in clear contrast to minimization problems where all stable stationary points of the gradient dynamics are local minima. The stability of these undesired stationary points is therefore an additional difficulty that one has to consider for escaping from such saddles. 
While a standard trick for escaping saddles in minimization problems consists of adding a small perturbation to the gradient, we will demonstrate that this does not guarantee avoiding undesired stationary points.  

Throughout the paper, we will refer to a \emph{desired} local saddle point as a local minimum in $\pmin$ and maximum in $\pmax$. This characterization implies that the Hessian matrix at $(\pmin,\pmax)$ does not have a negative curvature direction  in $\pmin$ (which corresponds to an eigenvector of $\nabla^2_\pmin f$ with a negative associated eigenvalue) and a positive curvature direction in $\pmax$ (which corresponds to an eigenvector of $\nabla^2_\pmax f$ with a positive associated eigenvalue). In that regard,  curvature information can be used to certify whether the desired min-max structure is met.

In this work, we propose the first saddle point optimization that exploits curvature to guide the gradient trajectory towards the desired saddle points that respect the min-max structure. Since our approach only makes use of the eigenvectors corresponding to the maximum and minimum eigenvalue (rather than the whole eigenspace), we will refer to it as \emph{extreme} curvature exploitation.
We will prove that this type of curvature exploitation avoids convergence to undesired saddles--albeit not guarantees convergence on a general non-convex-concave saddle point problem. Our contribution is linked to the recent research area of stability analysis for gradient-based optimization in general saddle point problems. Nagarajan et al. \cite{nagarajan2017gradient} have shown that the gradient method is stable at locally optimal saddles. Here, we complete the picture by showing that this method is unfavourably stable at some points that are not locally optimal. Our empirical results also confirm the advantage of curvature exploitation in saddle point optimization.

\section{RELATED WORK}
\label{sec:related_work}

\paragraph{Asymptotical Convergence}

In the context of optimizing a Lagrangian, the pioneering works of~\cite{kose1956solutions, arrow1958studies} popularized the use of the primal-dual dynamics to arrive at the saddle points of the objective. The work of~\cite{arrow1958studies} analyzed the stability of this method in continuous time proving global stability results under strict convex-concave assumptions. This result was extended in~\cite{uzawa1958iterative} for a discrete-time version of the subgradient method with a constant step size rule, proving  that the iterates converge to a neighborhood of a saddle point. Results for a decreasing step size were provided in~\cite{golshtein1974generalized, maistroskii1977gradient} while \cite{nemirovskii1978cezare} analyzed an adaptive step size rule with averaged parameters.
The work of~\cite{cherukuri2017saddle} has shown that the conditions of the objective can be relaxed, proving asymptotic stability to the set of saddle points is guaranteed if either the convexity or concavity properties are strict, and convergence is pointwise. They also proved that the strictness assumption can be dropped under other linearity assumptions or assuming strongly joint quasiconvex-quasiconcave saddle functions.

However, for problems where the function considered is not strictly convex-concave, convergence to a saddle point is not guaranteed, with the gradient dynamics leading instead to oscillatory solutions~\cite{holding2014convergence}. These oscillations can be addressed by averaging the iterates~\cite{nemirovskii1978cezare} or using the extragradient method (a perturbed version of the gradient method)~\cite{korpelevich1976extragradient, gidel2018variational}.

There are also instances of saddle point problems that do not satisfy the various conditions required for convergence. A notable example are generative adversarial networks (GANs) for which the work of \cite{nagarajan2017gradient} proved local asymptotic stability under certain suitable conditions on the representational power of the two players (called discriminator and generator). Despite these recent advances, the convergence properties of GANs are still not well understood.

\paragraph{Non-asymptotical Convergence}
An explicit convergence rate for the subgradient method with a constant stepsize was proved in~\cite{nedic2009subgradient} for reaching an approximate saddle point, as opposed to asymptotically exact solutions. Assuming the function is convex-concave, they proved a sub-linear rate of convergence. Rates of convergence have also been derived for the extragradient method~\cite{korpelevich1976extragradient} as well as for mirror descent~\cite{nemirovski2004prox}.

In the context of GANs, \cite{nowozin2016f} showed that a single-step gradient method converges to a saddle point in a neighborhood around the saddle point in which the function is strongly convex-concave.
The work of~\cite{liang2018interaction} studied the theory of non-asymptotic convergence to a local Nash equilibrium. They prove that--assuming local strong convexity-concavity--simultaneous gradient descent achieves an exponential rate of convergence near a stable local Nash equilibrium. They also extended this result to other discrete-time saddle point dynamics such as optimistic mirror descent or predictive methods.

\paragraph{Negative Curvature Exploitation}
The presence of negative curvature in the objective function indicates the existence of a potential descent direction, which is commonly exploited in order to escape saddle points and reach a local minimizer. Among these approaches are trust-region methods that guarantee convergence to a second-order stationary point~\cite{conn2000trust, nesterov2006cubic, cartis2011adaptive}. While a na\"ive implementation of these methods would require the computation and inversion of the Hessian of the objective, this can be avoided by replacing the computation of the Hessian by Hessian-vector products that can be computed efficiently in $\bigo(nd)$~\cite{Pearlmutter94fastexact}. This is applied e.g. using matrix-free Lanczos iterations~\cite{curtis2017exploiting} or online variants such as Oja's algorithm~\cite{allen2017natasha}. Sub-sampling the Hessian can furthermore reduce the dependence on $n$ by using various sampling schemes~\cite{kohler2017sub,xu2017newton}. Finally, ~\cite{allen2017neon2, xu2017first} showed that first-order information can act as a noisy Power method allowing to find a negative curvature direction.

In contrast to these classical results that "blindly" try to escape any type of saddle-point, our aim is to exploit curvature information to reach a specific type of stationary point that satisfies the min-max condition required at the optimum of the objective function.
\section{PRELIMINARIES}

\paragraph{Definition: Locally Optimal Saddles}
  Let us define a $\gamma$-neighbourhood around the point $(\pmin^*,\pmax^*)$ as 
\begin{align} 
 \K^*_\gamma  =\{(\pmin ,\pmax) \big| \: \;  \| \pmin - \pmin^* \| \leq \gamma, \| \pmax - \pmax^* \| \leq \gamma \}
\end{align} 
 with a sufficiently small $\gamma>0$. Throughout the paper, we follow a common approach, see e.g. \cite{mescheder2017numerics, DBLP:journals/corr/NagarajanK17}, for this type of problem and relax the condition of Eq. \eqref{eq:global_saddle_point} to hold only in a local neighbourhood.
\begin{define} \label{def:local_saddle_point}
The point $(\pmin^*,\pmax^*)$ is a locally optimal saddle point of the problem in Eq.~\eqref{eq:saddle-problem} if 
\begin{align} 
f(\pmin^*,\pmax) \leq f(\pmin^*, \pmax^*) \leq f(\pmin,\pmax^*) 
\end{align} 
holds for $\forall (\pmin,\pmax) \in \K^*_\gamma$.
\end{define}

\paragraph{Assumptions}
For the sake of further analysis, we require the function $f$ to be sufficiently smooth, and its second order derivatives with respect to the parameters $\pmin$ and $\pmax$ to be non-degenerate at the optimum $(\pmin^*,\pmax^*)$.

\begin{assumption}[Smoothness]\label{ass:smoothness} 
We assume that $f(\z) := f(\pmin,\pmax)$ is a $C^2$ function, and that its gradient and Hessian are Lipschitz with respect to the parameters $\pmin$ and $\pmax$, i.e. we assume that the following inequalities hold:
\begin{align} \label{assum:Lipschitzness}
\| \nabla f (\z)  - \nabla f(\widetilde{\z}) \| &\leq L_\z \| \z - \widetilde{\z} \| \\
\| \nabla^2 f (\z)  - \nabla^2  f(\widetilde{\z}) \| &\leq \rho_\z  \| \z  - \widetilde{\z} \| \\
\| \nabla_\pmin f (\z)  - \nabla_\pmin f(\widetilde{\z}) \| &\leq L_\pmin \| \z - \widetilde{\z} \| \\
\| \nabla^2_\pmin f (\z)  - \nabla^2_\pmin  f(\widetilde{\z}) \| &\leq \rho_\pmin  \| \z  - \widetilde{\z} \| \\ 
\| \nabla_\pmax f (\z)  - \nabla_\pmax f(\widetilde{\z}) \| &\leq L_\pmax \| \z - \widetilde{\z} \| \\
\| \nabla^2_\pmax f (\z)  - \nabla^2_\pmax  f(\widetilde{\z}) \| &\leq \rho_\pmax  \| \z  - \widetilde{\z} \| 
\end{align}
Moreover, we assume bounded gradients, i.e.
\begin{align}
\| \nabla_\pmin f(\z) \| \leq \ell_\pmin, \; \| \nabla_\pmax f(\z) \| \leq& \ell_\pmax, \; \| \nabla_\z f(\z) \| \leq \ell_\z
\end{align} 

\end{assumption}  
\begin{assumption}[Non-degenerate Hessian at $(\pmin^*,\pmax^*)$] \label{assum:degeneracy}
We assume that the matrices $\nabla^2_\pmin f(\pmin^*,\pmax^*)$ and $\nabla^2_\pmax f(\pmin^*,\pmax^*)$ are non-degenerate for all locally optimal points $(\pmin^*, \pmax^*) \in \R^{k+d}$ as defined in Def. \ref{def:local_saddle_point}. 
\end{assumption}

With the use of Assumption~\ref{assum:degeneracy}, we are able to establish sufficient conditions on $(\pmin^*, \pmax^*)$ to be a locally optimal saddle point.

\begin{lemma} \label{lem:locally-optimal-saddle}
Suppose that $f$ satisfies assumption~\ref{assum:degeneracy}; then,  $\z^* := (\pmin^*,\pmax^*)$ is a locally optimal saddle point on $\K_\gamma^*$ if and only if the gradient with respect to $\z^*$ is zero, i.e. 
\begin{align} 
\nabla f(\pmin^*,\pmax^*) = 0,  
\end{align}
and the second derivative at $(\pmin^*,\pmax^*)$ is positive definite in $\pmin$ and negative definite in $\pmax$ \footnote{In the game theory literature, such point is commonly referred to as local Nash equilibrium, see e.g.~\cite{liang2018interaction}.}, i.e. there exist $\mu_\pmin,\mu_\pmax>0$ such that 
\begin{align} 
\nabla^2_{\pmin} f(\pmin^*,\pmax^*) \succ \mu_\pmin \I , \qquad \nabla^2_{\pmax} f(\pmin^*,\pmax^*) \prec  - \mu_\pmax \I.
\end{align} 

\end{lemma}

\section{UNDESIRED STABILITY}

\paragraph{Asymptotic Scenarios}
There are three different asymptotic scenarios for the gradient iterations in Eq.~\eqref{eq:saddle_gradient}: (i) divergence (i.e. $ \lim_{t\to \infty} \|(\pmin_t,\pmax_t)\| \to \infty$), (ii) being trapped in a loop (corresponding to $\lim_{t\to \infty} \| \nabla f \| > 0$), and (iii) convergence to a stationary point of the gradient updates (i.e. $\lim_{t\to \infty} \| \nabla f \| = 0 $). To the best of our knowledge, there is no convergence guarantee for general saddle point optimization. Typical convergence guarantees require convexity-concavity or somewhat relaxed conditions such as quasiconvexity-quasiconcavity of $f$ \cite{cherukuri2017saddle}. This paper focuses on the third outline case and investigates the theoretical guarantees for a convergent series. We will show that gradient-based optimization can converge to some \textit{undesired} stationary points and propose an optimizer that uses extreme curvature information to alleviate this problem. We specifically highlight that we do not provide any convergence guarantee. Rather, we investigate if a convergent sequence is guaranteed to yield a valid solution to the local saddle point problem, i.e., if it always converges to a locally optimal saddle point as defined in Def. \ref{def:local_saddle_point}.


\paragraph{Local Stability}
A stationary point of the gradient iterations can be either stable or unstable. The notion of stability characterizes the behavior of the gradient iterations in a local region around the stationary point. In the neighborhood of a stable stationary point, successive iterations of the method are not able to escape the region. Conversely, we consider a stationary point to be unstable if it is not stable \cite{khalil2002nonlinear}.
The stationary point $\z^* = (\pmin^*,\pmax^*)$ (for which  $\nabla f(\z^*) = 0$ holds) is a locally stable point of the gradient iterations in Eq. \ref{eq:saddle_gradient}, if the Jacobian of its dynamics has only eigenvalues $\lambda_i$ within the unit disk, i.e. 
\begin{align} 
\left| \lambda_i \left( \I + \eta \begin{bmatrix}
-\nabla^2_\pmin f(\z^*) & -\nabla^2_{\pmin,\pmax} f(\z^*) \\ 
\nabla^2_{\pmax,\pmin} f(\z^*) & \nabla^2_{\pmax} f(\z^*)  \end{bmatrix} \right) \right| \leq 1.
\end{align}

\begin{define}[Stable Stationary Point of Gradient Dynamics]
 A point $\z = (\pmin, \pmax)$ is a stable stationary point of the gradient dynamics in Eq. \eqref{eq:saddle_gradient} (for an arbitrarily small step size $\eta > 0$) if $\nabla_\z f = 0$ and if the eigenvalues of the matrix  
 \begin{align}\label{stability-matrix}
     \begin{bmatrix}
-\nabla^2_\pmin f(\z) & -\nabla^2_{\pmin,\pmax} f(\z) \\ 
\nabla^2_{\pmax,\pmin} f(\z) & \nabla^2_{\pmax} f(\z)  \end{bmatrix}
 \end{align}
 only have eigenvalues with negative real-part.
\end{define}

\paragraph{Random Initialization}
In the following, we will use the notion of stability to analyze the asymptotic behavior of the gradient method. We start with a lemma extending known results for general minimization problems that prove that gradient descent with random initialization almost surely converges to a stable stationary point~\cite{lee2016gradient}.
\begin{lemma}[Random Initialization] \label{lemma:random_initialization}
Suppose that assumptions~\ref{ass:smoothness} and \ref{assum:degeneracy} hold. Consider the gradient iterate sequence of Eq.~\eqref{eq:saddle_gradient} with step size $\eta < \min\left( \frac{1}{L_x}, \frac{1}{L_y} , \frac{1}{\sqrt{2}L_z}\right)$ starting from a random initial point. If the sequence converges to a stationary point, then the stationary point is almost surely stable. 
\end{lemma}

\paragraph{Undesired Stable Stationary Point}
If all stable stationary points of the gradient dynamics would be locally optimal saddle points, then the result of Lemma~\ref{lemma:random_initialization} guarantees almost sure  convergence to a solution of the saddle point problem in Eq.~\eqref{eq:saddle-problem}. Previous work by \cite{mescheder2017numerics, nagarajan2017gradient} has shown that every locally optimal saddle point is a stable stationary point of the gradient dynamics. While for minimization problems, the set of stable stationary points is the same as the set of local minima, this might not be the case for the problem we consider here. Indeed, the gradient dynamics might introduce additional stable points that are not locally optimal saddle points. We illustrate this claim in the next example.

\paragraph{Example } Consider the following two-dimensional saddle point problem\footnote{To guarantee smoothness, one can restrict the domain of $f$ to a bounded set.}
\begin{align}\label{eq:conv_conv_function}
    \min_{x \in \R} \max_{y\in \R} \left[f(x,y) = 2 x^2 + y^2 + 4xy + \frac{4}{3}y^3 - \frac{1}{4}y^4 \right]
\end{align}
with $x,y \in \R$. 
The critical points of the function, i.e. points for which $\nabla f(x,y) = 0$, are
\begin{align}
    \begin{split}
        \z_0 = (0,0) \qquad \z_1 &= (-2-\sqrt{2},2+\sqrt{2}) \\
    \z_2 &= (-2+\sqrt{2},2-\sqrt{2})
    \end{split}
\end{align}
Evaluating the Hessians at the three critical points gives rise to the following three matrices:
\begin{align}
    \begin{split}
        \H(\z_0) = \mymatrix{4}{4}{4}{2} \qquad \H(\z_1) &= \mymatrix{4}{4}{4}{-4\sqrt{2}}  \\ 
        \H(\z_2) &= \mymatrix{4}{4}{4}{4\sqrt{2}}.
    \end{split}
\end{align}
We see that only $\z_1$ is a locally optimal saddle point, namely that $\nabla_\pmin^2f(\z_1) = 4 > 0$ and $\nabla_\pmax^2f(\z_1) = -4\sqrt{2} < 0$, whereas the two other points are both a local minimum in the parameter $\pmax$, rather than a maximum. However, figure \ref{fig:convconv_example_a} illustrates gradient steps converging to the undesired stationary point $\z_0$ because it is a locally stable point of the dynamics\footnote{This can be easily shown by observing that the real-part of the eigenvalues of the matrix in Eq. \ref{stability-matrix}, evaluated at $\z_0$, are all negative.}. Hence, even small perturbations of the gradients in each step can not avoid convergence to this point (see Figure \ref{fig:convconv_example_stochastic}). 

\begin{figure*}
    \centering
    \begin{subfigure}[t]{0.4\textwidth}
        \centering
        \includegraphics[width=1\textwidth]{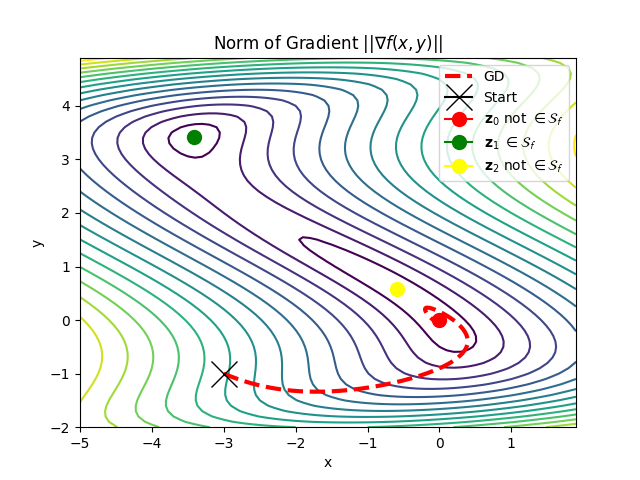}
        \caption{The contour lines correspond to the norm of the gradient $\| \nabla f(x,y) \|$.}
        \label{fig:convconv_example_a}
    \end{subfigure}%
    ~ 
    \begin{subfigure}[t]{.4\textwidth}
        \centering
        \includegraphics[width=1\textwidth]{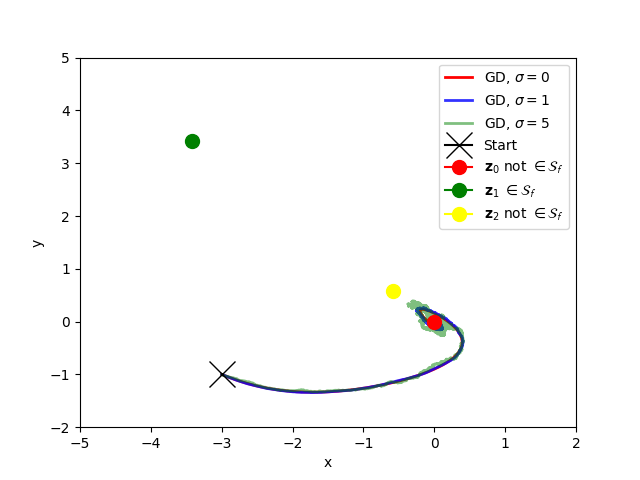}
        \caption{Optimization trajectory when adding Gaussian noise from $\mathcal{N}(0, \sigma)$ to the gradient in every step.}
        \label{fig:convconv_example_stochastic}
    \end{subfigure}
    \caption{Optimization trajectory of the gradient method on the function $f(x,y)$ in Eq.~\eqref{eq:conv_conv_function} with a step size of $\eta = 0.001$. The method converges to the critical point (0,0), even though it is not a locally optimal saddle point, and therefore not a solution to the problem defined in \eqref{eq:conv_conv_function}. We denote the set of locally optimal saddle points of the function $f$ with $\mathcal{S}_f$.}
    \label{fig:convconv_example}
\end{figure*}

\section{EXTREME CURVATURE EXPLOITATION}
The previous example has shown that gradient iterations on the saddle point problem introduce undesired stable points. In this section, we propose a strategy to escape from these points. Our approach is based on exploiting curvature information as in \cite{curtis2017exploiting}.

\paragraph{Extreme Curvature Direction}
Let $\lambda_{\pmin}$ be the minimum eigenvalue of $\nabla^2_\pmin f(\z)$ with its associated eigenvector $\mathbf{v}_\pmin$, and $\lambda_{\pmax}$ be the maximum eigenvalue of $\nabla^2_\pmax f(\z)$ with its associated eigenvector $\mathbf{v}_\pmax$. Then, we define 
\begin{align}\label{eq:def:negative_curvature}
    \v_\z^{(-)} &=  \ind_{\{\lambda_\pmin < 0 \}}\frac{\lambda_{\pmin}}{2\rho_{\pmin}} \text{sgn}(\v_{\pmin}^\top \nabla_{\pmin}f(\z)) \v_{\pmin} \\
    \v_\z^{(+)} &= \ind_{\{\lambda_\pmax > 0 \}} \frac{\lambda_{\pmax}}{2\rho_\pmax} \text{sgn}(\v_{\pmax}^\top \nabla_{\pmax}f(\z))  \v_{\pmax}
\end{align}
where $\text{sgn}: \R \to \{-1, 1\}$ is the sign function.
Using the above vectors, we define $\v_\z:=(\v_\z^{(-)},\v_\z^{(+)})$  as the \emph{extreme} curvature direction at $\z$. 

\paragraph{Algorithm} Using the extreme curvature direction, we modify the gradient steps as follows:
\begin{align}\label{eq:cuvature_gradient_iterates}
    (\pmin_{t+1},\pmax_{t+1}) &= (\pmin_t,\pmax_t) + \v_{\z_t} + \eta (-\nabla_\pmin f_t , \nabla_\pmax f_t) \\
    f_t :&= f(\pmin_t,\pmax_t). \nonumber
\end{align}
This new update step is constructed by adding the extreme curvature direction to the gradient method of Eq.~\eqref{eq:saddle_gradient}. From now on, we will refer to this modified update as the \method{} (curvature exploitation for the saddle point problem) method. Note that the algorithm reduces to gradient-based optimization in regions where there are only positive eigenvalues in $\pmin$ and negative eigenvalues in $\pmax$ as the extreme curvature vector is zero. This includes the region around any locally optimal saddle point.

\begin{figure*}[h]
    \centering
    \includegraphics[width=.8\linewidth]{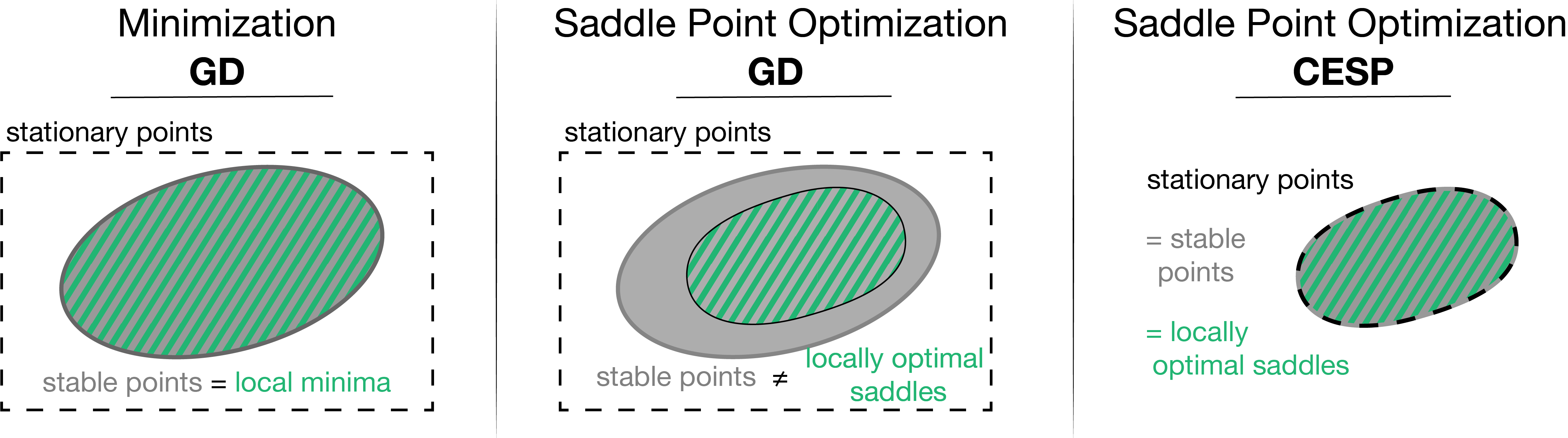}
    \caption{This Venn diagram summarizes the issue of gradient-based optimization that introduces \textit{undesired} stable points. This problem is unique to saddle point optimization and solved by the proposed \method{} method. [Best viewed in color.]}
    \label{fig:venn}
\end{figure*}

\paragraph{Stability}
Extreme curvature exploitation has already been used for escaping from unstable stationary points (i.e. saddle points) of gradient descent for minimization problems~\cite{curtis2017exploiting}. In saddle point problems, curvature exploitation is advantageous not only for escaping \textit{unstable} stationary points but also for escaping undesired \textit{stable} stationary points of the gradient iterates. The upcoming two lemmas prove that the set of stable stationary points of the \method{}{} dynamics and the set of locally optimal saddle points are the same -- therefore, the optimizer only converges to a solution of the local saddle point problem. The issue of gradient based optimization as well as the theoretical guarantees of the \method{} method are visualized in the Venn diagram in figure \ref{fig:venn}: while for minimization problems the set of stable points of gradient descent equals the set of local minima, we see that for saddle point problems gradient-based optimization introduces additional stable points outside of the set of locally optimal solutions. However, by exploiting extreme curvatures with our proposed \method{} method, all points outside of the set of locally optimal saddles become non-stationary. Hence, every convergent sequence of the \method{} method yields a solution to the local saddle point problem.

\begin{lemma} \label{lem:stationary_points_are_optimal}
A point $\z:=(\pmin,\pmax)$ is a stationary point of the iterates in Eq.~\eqref{eq:cuvature_gradient_iterates} if and only if $\z$ is a locally optimal saddle point. 
\end{lemma} 

We can conclude from the result of Lemma \ref{lem:stationary_points_are_optimal} that every stationary point of the \method{} dynamics is a locally optimal saddle point. The next Lemma establishes the stability of these points.

\begin{lemma} \label{lem:sp_are_stable_cegd}
Suppose that assumptions~\ref{ass:smoothness} and~\ref{assum:degeneracy} hold. Let $\z^* := (\pmin^*,\pmax^*)$ be a locally optimal saddle point, i.e.  
\begin{align} 
\nabla f(\z) = 0, \; \nabla^2_\pmin f(\z^*) \succeq \mu_\pmin \I, \;  \nabla^2_\pmax f(\z^*) \preceq - \mu_\pmax \I, \\
(\mu_\pmin, \mu_\pmax >0)\nonumber
\end{align}
Then the iterates of Eq.\eqref{eq:cuvature_gradient_iterates} are stable in $\K_\gamma^*$ for
\begin{align} 
\gamma \leq \min \{ \mu_\pmin/(\sqrt{2} \rho_\pmin),\mu_\pmax/(\sqrt{2} \rho_\pmax) \}.
\end{align} 
\end{lemma}
\paragraph{Escaping From Undesired Saddles}
Extreme curvature exploitation allows us to escape from undesired saddles. In the next lemma, we show that the optimization trajectory of \method{} stays away from all undesired stationary points of the gradient dynamics.
\begin{lemma} \label{lemma:escaping}
Suppose that $\z^*:= (\pmin^*,\pmax^*)$ is an undesired stationary point of the gradient dynamics, namely
\begin{align} 
\nabla f(\z^*) = 0, \| \v_{\z^*} \| > 0. 
\end{align}
Consider the iterates of Eq.~\eqref{eq:cuvature_gradient_iterates} starting from $\z_0 = (\pmin_0,\pmax_0)$ in a $\gamma$-neighbourhood of $\z^*$. After one step the iterates escape the $\gamma$-neighbourhood of $\z^*$, i.e.
 \begin{align} 
  \| \z_1 - \z^* \| \geq \gamma
 \end{align} 
 for a sufficiently small $\gamma = \bigo( \| \v_{\z^*}\| )$. 
\end{lemma}


\paragraph{Implementation with Hessian-vector products}\label{par:Hess_vec_prod}
Since storing and computing the Hessian in high dimensions is very costly, we need to find a way to efficiently extract the extreme curvature direction. The most common approach for obtaining the eigenvector corresponding to the largest absolute eigenvalue, (and the eigenvalue itself) of $\nabla^2_\pmin f(\z)$ is to run power iterations on $\I - \beta \nabla^2_\pmin f(\z)$ as 
\begin{align}
    \v_{t+1} = (\I - \beta \nabla^2_\pmin f(\z)) \v_t
\end{align}
where $\v_0$ is a random vector and $\v_{t+1}$ is normalized after every iteration. The parameter $\beta > 0$ is chosen such that $\I - \beta \nabla^2_\pmin f(\z) \succeq 0$. Since this method only requires implicit Hessian computation through a Hessian-vector product, it can be implemented as efficiently as gradient evaluations \cite{Pearlmutter94fastexact}. The results of \cite{doi:10.1137/0613066} provide a bound on the number of required iterations to extract the extreme curvature: for the case $\lambda_{\min}(\nabla^2_\pmin f(\z)) \leq - \gamma$,  $\frac{1}{\gamma}\log(k/\delta^2) L_x$ iterations suffice to find a vector $\hat{\v}$ such that $\hat{\v}^\top \nabla^2_\pmin f(\z) \hat{\v} \leq -\frac{\gamma}{2}$ with probability $1- \delta$ (cf. \cite{lee2016gradient}).

\paragraph{Comparison to second-order optimization}
We would like to draw the attention of the reader to the fact that the \method{} method only uses \emph{extreme} curvature which makes it conceptually different from second-order Newton-type optimization. Although there is a rich literature on second-order optimization for variational inequalities and convex-concave saddle point problems, to the best of our knowledge, there is neither theoretical nor practical evidence for success of these methods on general smooth saddle point problems. 


\section{CURVATURE EXPLOITATION FOR LINEAR-TRANSFORMED GRADIENT STEPS}
\paragraph{Linear-Transformed Gradient Optimization }
Applying a linear transformation to the gradient updates is commonly used to accelerate optimization for various types of problems. The resulting updates can be written in the general form 
\begin{align} \label{eq:linear_transformated_mapping}
\z_{t+1} = \z_t + \eta \A_{\z_t} (-\nabla_\pmin f_t, \nabla_\pmax f_t),\; f_t = f(\z_t) 
\end{align}
where $\A_{\z} = \mymatrix{\mathcal{A}}{0}{0}{\mathcal{B}}$ is a symmetric, block-diagonal $((k+d)\times(k+d))$-matrix. Different optimization methods use a different linear transformation $\A_{\z}$. Table \ref{tab:upd-matrices} in section \ref{sec:transformed_updates} in the appendix illustrates the choice of $\A_{\z}$ for different optimizers. {\sc Adagrad} \cite{Duchi:EECS-2010-24}, one of the most popular optimization methods in machine learning, belongs to this category. 

\paragraph{Extreme Curvature Exploitation} 
We can adapt \method{} to the linear-transformed variant:
\begin{align} \label{eq:curvatuer_linear_transformated_gradient}
\z_{t+1} = \z_t + \v_{\z_t} + \eta \A_{\z_t} (-\nabla_\pmin f_t, \nabla_\pmax f_t). 
\end{align} 
where we choose the linear transformation matrix $\A_{\z_t}$ to be positive definite. This variant of \method{} is also able to filter out the undesired stable stationary points of the gradient method for the saddle point problem. The following lemma proves that it has the same properties as the non-transformed version.

\begin{lemma}\label{lem:linear_transformed_guarantee}
    The set of locally optimal saddle points as defined in Def. \ref{def:local_saddle_point} and the set of stable points of the linear-transformed \method{} update method in Eq.~\eqref{eq:curvatuer_linear_transformated_gradient} are the same.
\end{lemma}
A direct implication of Lemma~\ref{lem:linear_transformed_guarantee} is that we can also use curvature exploitation for {\sc Adagrad}. Later, we will experimentally show the advantage of using curvature exploitation for this method. 

\section{EXPERIMENTS}

\subsection{Escaping From Undesired Stationary Points of the Toy Example} Previously, we saw that for the two dimensional saddle point problem on the function of Eq.~\eqref{eq:conv_conv_function}, gradient iterates may converge to an undesired stationary point that is not locally optimal. As shown in Figure \ref{fig:convconv_example_with_cegd}, \method{} solves this issue. In this example, simultaneous gradient iterates converge to the undesired stationary point $\z_0 = (0,0)$ for many different initialization parameters, whereas our method always converges to the locally optimal saddle point. A plot of the basin of attraction of the two different optimizers on this example is presented in Figure \ref{fig:convconv_example_basin_attraction} in the appendix.

\begin{figure*}[h]
    \vspace{-3mm}
    \centering
    \begin{subfigure}[t]{0.4\textwidth}
        \centering
        \includegraphics[width=1\textwidth]{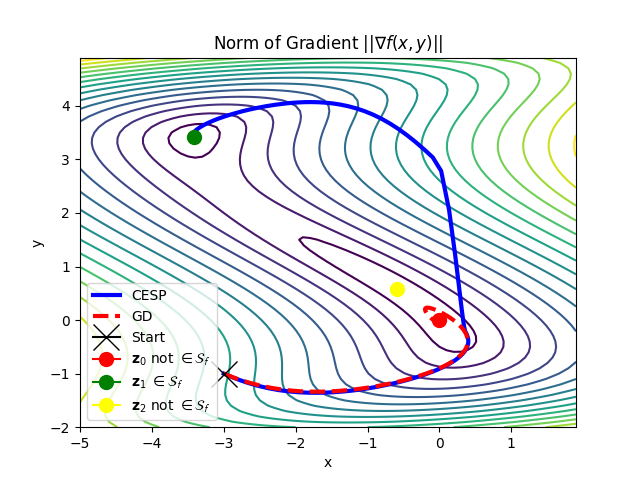}
        \caption{}
    \end{subfigure}%
    ~ 
    \begin{subfigure}[t]{.4\textwidth}
        \centering
        \includegraphics[width=1\textwidth]{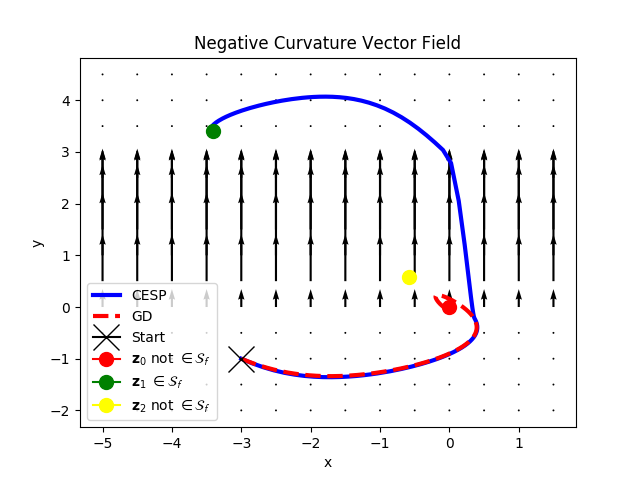}
        \caption{}
    \end{subfigure}
    \caption{Comparison of the trajectory of GD and the \method{} proposal on the function in Eq.~\eqref{eq:conv_conv_function} from the starting point $(-3, -1)$. The background in the right plot shows the vector field of the extreme curvature, as defined in Eq.~\eqref{eq:def:negative_curvature}. Note that for the function in Eq.~\eqref{eq:conv_conv_function} the curvature in the $x$-dimension is constant and positive, and therefore $\v_\z^{(-)}$ is always zero.}
    \label{fig:convconv_example_with_cegd}
\end{figure*}

\subsection{Robust Optimization}
Although robust optimization \cite{ben2009robust} is often formulated as a convex-concave saddle point problem, we consider robust optimization on neural networks that do not fulfill this assumption. The optimization problem that we target here is an application of robust optimization in empirical risk minimization \cite{NamkoongD17}, namely solving
\begin{multline}
    \min_{\pmin} \sup_{P \in \mathcal{P}} \big[ f(\mathbf{X}; \pmin, \mathcal{P}) \\= \left\{ \mathbb{E}_{P}[l(\mathbf{X}; \theta)] : D(P \| \hat{P}_n) \leq \frac{\rho}{n}\right\} \big]
\end{multline}
where $l(X; \pmin)$ denotes the cost function to minimize, $\mathbf{X}$ the data, and $D(P \| \hat{P}_n)$ a divergence measure between the true data distribution $P$ and the empirical data distribution $\hat{P}_n$. \\
We use this framework on the Wisconsin breast cancer data set, which is a binary classification task with 30 attributes and 569 samples, and choose a multilayer perceptron with a non-convex sigmoid activation as the classifier. Due to the relatively small sample size, we can compute the gradient exactly in this case. We choose the objective $f(X; \pmin, \mathcal{P})$ in this setting to be
\begin{align}
    f(X; &\pmin, \mathcal{P})\nonumber \\
    = &-\sum_{i=1}^n p_i^* \left[ y_i \log(\hat{y}(\mathbf{X}_i)) + (1-y_i) \log(1-\hat{y}(\mathbf{X}_i))\right] \nonumber\\
    &- \lambda \sum_{i=1}^n (p^*_i - \frac{1}{n})^2
\end{align} 
where we add a regularization term with $\lambda > 0$ to enforce the divergence constraint. Figure \ref{fig:robust_opt} shows the comparison of the gradient method (GD) and our CESP optimizer on this problem in terms of the minimum eigenvalue of $\nabla^2_\pmin f(\mathbf{X}; \pmin, \mathcal{P})$. Note that $f$ is concave with respect to $\mathcal{P}$ and therefore its Hessian is constant negative. The results indicate the anticipated behavior that \method{} is able to more reliably drive a convergent series towards a solution where the minimum eigenvalue of $\nabla^2_\pmin f(\mathbf{X}; \pmin, \mathcal{P})$ is positive.\\

\begin{figure}
    \vspace{-5mm}
    \centering
    \begin{subfigure}[t]{0.25\textwidth}
        \centering
        \includegraphics[width=1\textwidth]{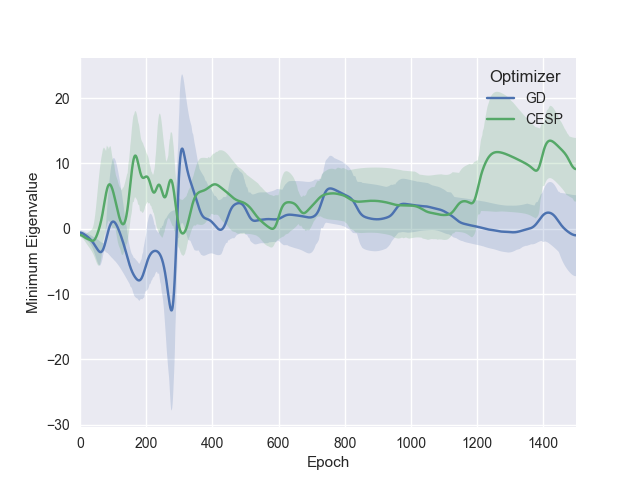}
    \end{subfigure}%
    ~ 
    \begin{subfigure}[t]{.25\textwidth}
        \centering
        \includegraphics[width=1\textwidth]{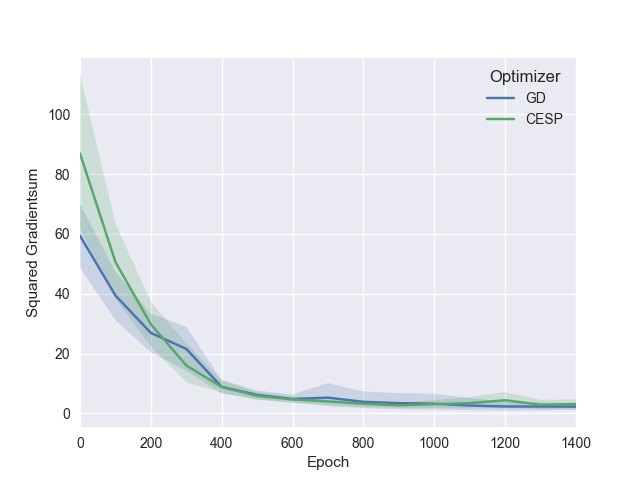}
    \end{subfigure}%
    \caption{The left plot shows the minimum eigenvalue of $\nabla^2_\pmin f(X; \pmin, \mathcal{P})$ and the right plot the squared gradient norm. The solid line shows the mean value whereas the blurred area indicates the 90th percentile. The \method{} optimzer is superior in this example because for convergent series where the gradientnorm $\approx 0$ the minimum eigenvalue of $\nabla^2_\pmin f(X; \pmin, \mathcal{P}) > 0$; for gradient-based optimization this is not the case.}
    \label{fig:robust_opt}
\end{figure}


\vspace{-3mm}
\subsection{Generative Adversarial Networks}
\label{sec:single_layer_GAN-exp}
This experiment evaluates the performance of the \method{} method for training a Generative Adversarial Network (GAN), which reduces to solving the saddle point problem 
\begin{multline}\label{eq:GAN_objective}
    \min_{\pmin} \max_\pmax \big[f(\pmin, \pmax) = \mathbb{E}_{\theta \sim p_d} \log(D_{\pmin}(\theta)) \\+ \mathbb{E}_{\mathbf{z} \sim p_z} \log(1 - D_{\pmin}(G_{\pmax}(\mathbf{z}))) \big]
\end{multline}

where the functions $D: \R^n \to [0, 1]$ and $G: \R^m \to \R^n$ are represented by neural networks parameterized with the variables $\pmin$ and $\pmax$, respectively. We use the MNIST data set and a simple GAN architecture with 1 hidden layer and 100 units. More details about the network architecture and parameters are summarized in table \ref{tab:model_params} in the Appendix.

We investigate the advantage of curvature exploitation for {\sc Adagrad}, which is a member of the class of linear-transformed gradient methods often used for saddle point problems. Moreover, we make use of Power iterations as described in section \ref{par:Hess_vec_prod} to efficiently approximate the extreme curvature vector. Note that since we're using mini batches in this experiment, we do not have access to the correct gradient information but also rely on an approximation here.\\
As before, we evaluate the efficacy of the negative curvature step in terms of the spectrum of $f$ at a (approximately) convergent solution $\z^*$. We compare \method{} to the vanilla {\sc Adagrad} optimizer. Since we are interested in a solution that gives rise to a locally optimal saddle point, we track (an approximation of) the smallest eigenvalue of $\nabla_{\pmin}^2 f(\z^*)$ and the largest eigenvalue of $\nabla_{\pmax}^2 f(\z^*)$ through the optimization. Using these estimates, we can evaluate if a method has converged to a locally optimal saddle point.
The results are shown in figure \ref{fig:GAN_exp}. The decrease in terms of the squared norm of the gradients indicates that both methods converge to a solution. Moreover, both fulfill the condition for a locally optimal saddle point for the parameter $\pmax$, i.e. the maximum eigenvalue of $\nabla_{\pmax}^2 f(\z^*)$ is negative. However, the graph of the minimum eigenvalue of $\nabla_{\pmin}^2 f(\z^*)$ shows that \method{} converges faster, and with less frequent and severe spikes, to a solution where the minimum eigenvalue is zero. Hence, the negative curvature step seems to be able to drive the optimization procedure to regions that yield points closer to a locally optimal saddle point. 

Even though this empirical result highlights the benefits of using curvature exploitation, we observe zero eigenvalues of $\nabla_{\pmin}^2 f(\z^*)$ for convergent solutions, which violates the conditions required for our analysis. This observation is in accordance with recent empirical evidence \cite{sagun2016eigenvalues} showing that the Hessian is actually degenerate for common deep learning architectures. The phenomenon itself as well as approaches to address it are left as future work. One potential direction would be to investigate if high-order derivatives could be used at points where the Hessian is degenerate.

\begin{figure}[h]
    \centering
    \begin{subfigure}[t]{0.25\textwidth}
        \centering
        \includegraphics[width=1\textwidth]{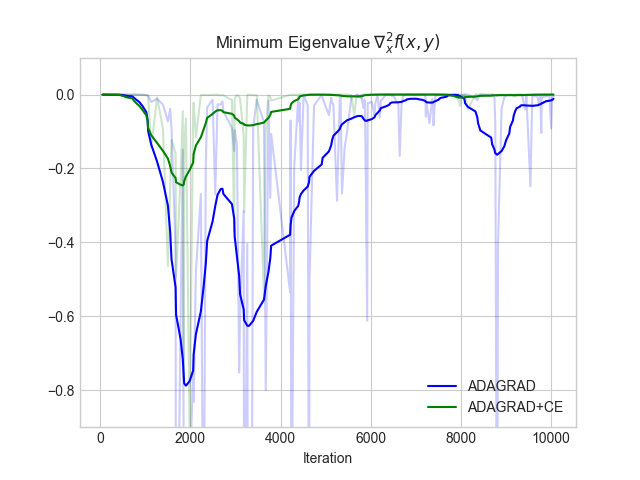}
    \end{subfigure}%
    ~ 
    \begin{subfigure}[t]{.25\textwidth}
        \centering
        \includegraphics[width=1\textwidth]{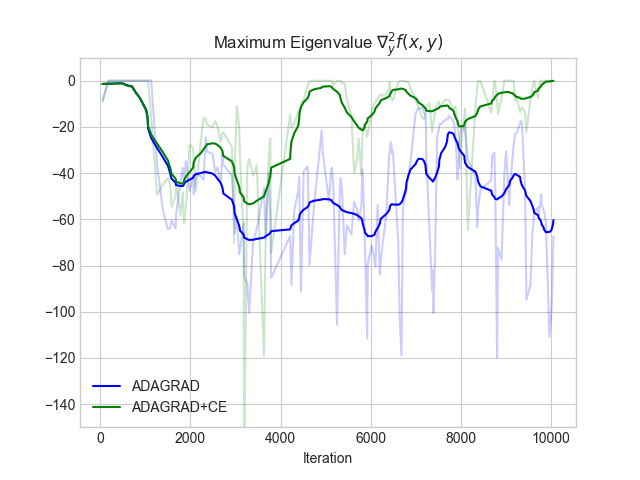}
    \end{subfigure}%
    \\
    \begin{subfigure}[t]{.25\textwidth}
        \centering
        \includegraphics[width=1\textwidth]{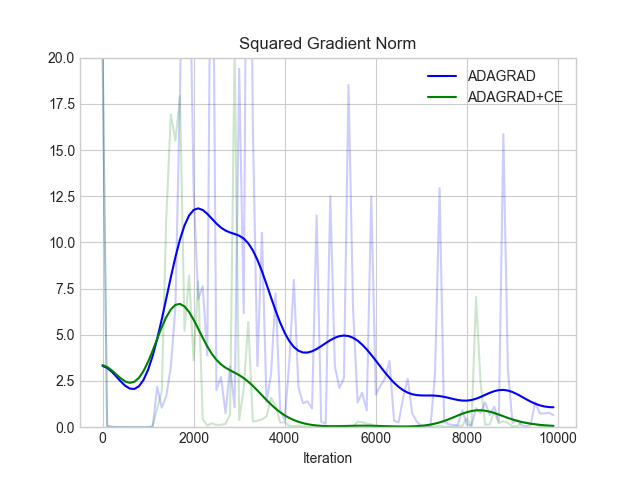}
    \end{subfigure}
    \caption{The first two plots show the minimum eigenvalue of $\nabla_{\pmin}^2 f(\pmin, \pmax)$ and the maximum eigenvalue of $\nabla_{\pmax}^2 f(\pmin, \pmax)$, respectively. The third plot shows $\|\nabla f(\z_t)\|^2$. The transparent graph shows the raw computed values, whereas the solid graph is smoothed with a Gaussian filter.}
    \label{fig:GAN_exp}
\end{figure}



\section{CONCLUSION}
We focused our study on reaching a solution to the local saddle point problem. First, we have shown that gradient methods have stable stationary points that are not locally optimal, which is a problem exclusively arising in saddle point optimization. Second, we proposed a novel approach that exploits extreme curvature information to avoid the undesired stationary points. We believe this work highlights the benefits of using curvature information for saddle point problems and might open the door to other novel algorithms with stronger global convergence guarantees. 

\bibliography{local-saddle-opt}
\bibliographystyle{plain}
\clearpage
\onecolumn
\appendix

\begin{huge}\textbf{APPENDIX}\end{huge}
\section{Theoretical Analysis}
\subsection{Lemma \ref{lem:locally-optimal-saddle}}\label{proof:lem:locally-optimal-saddle}

\begin{lemma_custom_no}{\ref{lem:locally-optimal-saddle}}
Suppose that $f$ satisfies assumption~\ref{assum:degeneracy}; then,  $\z^* := (\pmin^*,\pmax^*)$ is a locally optimal saddle point on $\K_\gamma^*$ if and only if the gradient is zero, i.e. 
\begin{align} 
\nabla f(\pmin^*,\pmax^*) = 0,  
\end{align}
and the second derivative at $(\pmin^*,\pmax^*)$ is positive definite in $\pmin$ and negative definite in $\pmax$, i.e., there exist $\mu_\pmin,\mu_\pmax>0$ such that 
\begin{align} 
\nabla^2_{\pmin} f(\pmin^*,\pmax^*) \succ \mu_\pmin \I , \qquad \nabla^2_{\pmax} f(\pmin^*,\pmax^*) \prec  - \mu_\pmax \I.
\end{align} 
\end{lemma_custom_no}

\begin{proof}
From definition \ref{def:local_saddle_point} follows that a locally optimal saddle point $(\pmin^*, \pmax^*) \in \K_\gamma^*$ is a point for which the following two conditions hold:
\begin{align}
    f(\pmin^*, \pmax) \leq f(\pmin, \pmax) \quad \text{and} \quad f(\pmin, \pmax^*) \geq f(\pmin, \pmax) \quad \forall (\pmin,\pmax) \in \K_\gamma^*
\end{align}
Hence, $\pmin$ is a local minimizer of $f$ and $\pmax$ is a local maximizer. We therefore, without loss of generality, prove the statement of the lemma only for the minimizer $\pmin$, namely that
\begin{enumerate}[(i)]
    \item $\nabla_\pmin f(\pmin^*, \pmax) = 0 \quad \forall \pmax \text{ s.t. } \lVert \pmax - \pmax^* \rVert \leq \gamma$
    \item $\nabla^2_\pmin f(\pmin^*, \pmax) \succ \mu_\pmin \I \quad \forall \pmax \text{ s.t. } \lVert \pmax - \pmax^* \rVert \leq \gamma, \mu_\pmin > 0$.
\end{enumerate} 
The proof for the maximizer $\pmax$ directly follows from this.\\

\begin{enumerate}[(i)]
     \item If we assume that $\nabla_\pmin f(\pmin^*, \pmax) \neq 0$, then there exists a feasible direction $\mathbf{d} \in \R^k$ such that  $\nabla^\top_\pmin f(\pmin^*, \pmax) \mathbf{d} < 0$, and we can find a step size $\alpha > 0$ for $x(\alpha) = \pmin^* + \alpha \mathbf{d}$ s.t. $\alpha \|\mathbf{d}\| \leq \gamma$ with $\|\mathbf{d}\| = 1$.
    Using the smoothness assumptions (Assumption \ref{ass:smoothness}), we arrive at the following inequality
    \begin{multline}\label{equ:smoothness_inequ}
        \biggl| f(\pmin(\alpha), \pmax) - f(\pmin^*, \pmax) - \nabla^\top_\pmin f(\pmin^*, \pmax) (\pmin(\alpha) - \pmin^*) \\ - \frac{1}{2} (\pmin(\alpha) - \pmin^*)^\top \nabla^2_\pmin f(\pmin^*, \pmax) (\pmin(\alpha) - \pmin^*) \biggr|  \leq \frac{\rho_\pmin}{6} \| (\pmin(\alpha) - \pmin^*)\| 
    \end{multline}
    Hence, it holds that:
    \begin{align}
        f(\pmin(\alpha), \pmax) \leq f(\pmin^*, \pmax) + \alpha \left( \nabla^\top_\pmin f(\pmin^*, \pmax) \mathbf{d} + \frac{\rho_\pmin}{6} + \frac{1}{2} \alpha L_x \right)
    \end{align}
    By choosing the gradient descent direction $\mathbf{d} = -\beta \nabla^\top_\pmin f(\pmin^*, \pmax)$ (with $\beta > 0$ s.t. $\|\mathbf{d}\| = 1$), we can find a step size  $0 < \alpha <  \frac{2 \beta}{L_x} \|\nabla_\pmin f(\pmin^*, \pmax) \|^2 - \frac{\rho_\pmin}{3 L_x}$ such that $f(\pmin(\alpha), \pmax) < f(\pmin^*, \pmax)$,
    
    which contradicts that $f(\pmin^*, \pmax)$ is a local minimizer. Hence, $\nabla_\pmin f(\pmin^*, \pmax) = 0$ is a necessary condition for a local minimizer.
    \item To prove the second statement, we again make use of inequality \eqref{equ:smoothness_inequ} coming from the smoothness assumption and the update $\pmin(\alpha) = \pmin^* + \alpha \mathbf{d}$ s.t. $\alpha \|\mathbf{d}\| \leq \gamma$ with $\|\mathbf{d}\| = 1$. From (i) we know that $\nabla_\pmin f(\pmin^*, \pmax) = 0$ and, therefore, we obtain:
    \begin{align}
        &\left| f(\pmin(\alpha), \pmax) - f(\pmin^*, \pmax) - \frac{1}{2} \mathbf{d}^\top \nabla^2_\pmin f(\pmin^*, \pmax) \mathbf{d} \right|  \leq \frac{\rho_\pmin}{6} \alpha \\
        \Rightarrow & f(\pmin(\alpha), \pmax) \leq f(\pmin^*, \pmax) + \frac{1}{2} \mathbf{d}^\top \nabla^2_\pmin f(\pmin^*, \pmax) \mathbf{d} + \frac{\rho_\pmin}{6} \alpha
    \end{align}
    If $\nabla^2_\pmin f(\pmin^*, \pmax)$ is not positive semi-definite, then there exists at least one eigenvector $\mathbf{v}$ with negative curvature, i.e. $\mathbf{v}^\top \nabla^2_\pmin f(\pmin^*, \pmax)\mathbf{v} = -\epsilon < 0$. This implies that for $\alpha > \frac{1}{3\epsilon} \rho_\pmin$ following the curvature vector $\mathbf{v}$ decreases the function value, i.e., $f(\pmin(\alpha), \pmax) < f(\pmin^*, \pmax)$. This contradicts that $f(\pmin^*, \pmax)$ is a local minimizer which proves the sufficient condition 
    \begin{align}
        \nabla^2_\pmin f(\pmin^*, \pmax) \succ \mu_\pmin \I \quad \text{, with} \; \mu_\pmin > 0.
    \end{align}
\end{enumerate}
\end{proof}

\subsection{Lemma \ref{lemma:random_initialization}}
\label{proof:lemm:random_init}
The following Lemma \ref{lemma:diffeomorphism_gd} proves that the gradient-based mapping for the saddle point problem is a diffeomorphism which will be needed in the proof for Lemma \ref{lemma:random_initialization}.
\begin{lemma}\label{lemma:diffeomorphism_gd}
    Suppose that assumption \ref{ass:smoothness} holds; then the gradient mapping for the saddle point problem
    \begin{align}
        g(\pmin, \pmax) = (\pmin, \pmax) + \eta (-\nabla_{\pmin}f(\pmin, \pmax), \nabla_{\pmax} f(\pmin, \pmax)) 
    \end{align}
    with step size $\eta < \min\left( \frac{1}{L_x}, \frac{1}{L_y} , \frac{1}{\sqrt{2}L_z}\right)$ is a diffeomorphism.
\end{lemma}
\begin{proof}
The following proof is very much based on the proof of proposition 4.5 from \cite{lee2016gradient}.\\
A necessary condition for a diffeomorphism is bijectivity. Hence, we need to check that $g$ is (i) injective, and (ii) surjective for $\eta < \min\left( \frac{1}{L_x}, \frac{1}{L_y} , \frac{1}{\sqrt{2}L_z}\right)$.

\begin{enumerate}[(i)]
    \item Consider two points $\z := (\pmin, \pmax), \widetilde{\z}: = (\widetilde{\pmin}, \widetilde{\pmax}) \in \K_\gamma$ for which 
    \begin{align}
        g(\z) = g(\widetilde{\z}) 
    \end{align}
    holds. Then, we have that
    \begin{align}
        \z - \widetilde{\z} &= \eta \myvec{-\nabla_\pmin f(\z)}{\nabla_\pmax f(\z)} - \eta \myvec{-\nabla_\pmin f(\widetilde{\z})}{\nabla_\pmax f(\widetilde{\z})} \\
        &= \eta \myvec{-\nabla_\pmin f(\z) + \nabla_\pmin f(\widetilde{\z})}{\nabla_\pmax f(\z) - \nabla_\pmax f(\widetilde{\z})}.
    \end{align}
    Note that
    \begin{align}
        \| \nabla_\pmin f(\z) - \nabla_\pmin f(\widetilde{\z}) \| &\leq \| \nabla f(\z) - \nabla f(\widetilde{\z}) \| \leq L_z \| \z - \widetilde{\z} \| \\
        \| \nabla_\pmax f(\z) - \nabla_\pmax f(\widetilde{\z}) \| &\leq \| \nabla f(\z) - \nabla f(\widetilde{\z}) \| \leq L_z \| \z - \widetilde{\z} \|,
    \end{align}
    from which follows that
    \begin{align}
        \| \z - \widetilde{\z} \| &\leq \eta \sqrt{L_z^2 \| \z - \widetilde{\z} \|^2 + L_z^2 \| \z - \widetilde{\z} \|^2} \\
        &= \sqrt{2}\eta L_z \| \z - \widetilde{\z} \|.
    \end{align}
    For $0< \eta < \frac{1}{\sqrt{2}L_z}$ this means $\z = \widetilde{\z}$, and therefore $g$ is injective.
    
    \item We will show that $g$ is surjective by constructing an explicit inverse function for both optimization problems individually. As suggested by \cite{lee2016gradient}, we make use of the proximal point algorithm on the function $-f$ for the parameters $\pmin, \pmax$, individually. \\
    For the parameter $\pmin$ the proximal point mapping of $-f$ centered at $\widetilde{\pmin}$ is given by
    \begin{align}
        \pmin(\widetilde{\pmin}) = \arg \min_{\pmin} \underbrace{ \frac{1}{2}\| \pmin - \widetilde{\pmin} \|^2 - \eta f(\pmin, \pmax)}_{h(\pmin)}
    \end{align}
    Moreover, note that $h(\pmin)$ is strongly convex in $\pmin$ if $\eta < \frac{1}{L_x}$:
    \begin{align}
        (\nabla_\pmin h(\pmin) - \nabla_\pmin h(\widehat{\pmin}))^\top (\pmin - \widehat{\pmin}) = (\pmin - \eta \nabla_\pmin f(\pmin, \pmax) - \widehat{\pmin} + \eta \nabla_\pmin f(\widehat{\pmin}, \pmax))^\top (\pmin - \widehat{\pmin}) \\
        = \| \pmin - \widehat{\pmin} \|^2 - \eta (\nabla_\pmin f(\pmin, \pmax) - \nabla_\pmin f(\widehat{\pmin}, \pmax))^\top (\pmin - \widehat{\pmin}) \geq (1- \eta L_x) \| \pmin - \widehat{\pmin} \|^2
    \end{align}
    Hence, the function $h(\pmin)$ has a unique minimizer, given by
    \begin{align}
        &0 \stackrel{!}{=} \nabla_\pmin h(\pmin)  = \pmin - \widetilde{\pmin} - \eta \nabla_\pmin f(\pmin, \pmax) \\
        &\Rightarrow \widetilde{\pmin} = \pmin - \eta \nabla_\pmin f(\pmin, \pmax)
    \end{align}
    which means that there is a unique mapping from $\pmin$ to $\widetilde{\pmin}$ under the gradient mapping $g$ if $\eta < \frac{1}{L_x}$.\\
    The same line of reasoning can be applied to the parameter $\pmax$ with the negative proximal point mapping of $-f$ centered at $\widetilde{\pmax}$, i.e.
    \begin{align}
        \pmax(\widetilde{\pmax}) = \arg \max_{\pmax} \underbrace{ -\frac{1}{2}\| \pmax - \widetilde{\pmax} \|^2 - \eta f(\pmin, \pmax)}_{h(\pmax)}
    \end{align}
    Similarly as before, we can observe that $h(\pmax)$ is strictly concave for $\eta < \frac{1}{L_y}$ and that the unique minimizer of $h(\pmax)$ yields the $\pmax$ update step of $g$. This let's us conclude that the mapping $g$ is surjective for $(\pmin, \pmax)$ if $\eta < \min\left(\frac{1}{L_x}, \frac{1}{L_y}\right)$
\end{enumerate}
Observing that for $\eta < \frac{1}{L_\z}$, $g^{-1}$ is continuously differentiable concludes the proof that $g$ is a diffeomorphism. 
\end{proof}

\begin{lemma_custom_no}{\ref{lemma:random_initialization}}[Random Initialization]
Suppose that assumption~\ref{ass:smoothness} holds. Consider gradient iterates of Eq.~\eqref{eq:saddle_gradient} with step size $\eta < \min\left( \frac{1}{L_x}, \frac{1}{L_y} , \frac{1}{\sqrt{2}L_z}\right)$ starting from a random initial point. If the iterates converge to a stationary point, then the stationary point is almost surely stable. 
\end{lemma_custom_no}
\begin{proof}
    From lemma \ref{lemma:diffeomorphism_gd} follows that the gradient update from Eq.~\eqref{eq:saddle_gradient} for the saddle point problem is a diffeomorphism. The remaining part of the proof follows directly from theorem 4.1 from \cite{lee2016gradient}.
\end{proof}

\subsection{Lemma \ref{lem:stationary_points_are_optimal}}\label{proof:lem:stationary_points_are_optimal}
\begin{lemma_custom_no}{\ref{lem:stationary_points_are_optimal}}
    The point $\z:=(\pmin,\pmax)$ is a stationary point of the iterates in Eq.~\eqref{eq:cuvature_gradient_iterates} if and only if $\z$ is a locally optimal saddle point. 
\end{lemma_custom_no} 
\begin{proof}
    The point $\z^*$ is a stationary point of the iterates if and only if $\v_{\z^*} + \eta (-\nabla_{\pmin}f(\z^*), \nabla_{\pmax}f(\z^*)) = 0$. Let's consider w.l.o.g. only the stationary point condition with respect to $\pmin$, i.e.
    \begin{align}
        \v_{\z^*} &= \eta \nabla_\pmin f(\z^*)
    \end{align}
    We prove that the above equation holds only if $\nabla f(\z^*) = \v_{\z^*}= 0$. This can be proven by a simple contradiction; suppose that $\nabla f(\z^*) \neq 0$, then multiplying both sides of the above equation by $\nabla f(\z^*)$ yields
    \begin{align}
         \underbrace{\lambda_{\pmin^*}/(2\rho_{\pmin})}_{< 0} \underbrace{\text{sgn}(\v_{\pmin^*}^\top \nabla_{\pmin}f(\z^*)) \v_{\pmin^*}^\top \nabla_\pmin f(\z^*)}_{> 0} &= \eta \lVert \nabla_\pmin f(\z^*) \rVert^2
    \end{align}
     Since the left-hand side is negative and the right-hand side is positive, the above equation leads to a contradiction. Therefore, $\nabla f(\z^*) = 0$ and $\v_{\z^*}=0$. This means that $\lambda_{\pmin^*} \geq 0$ and $\lambda_{\pmax^*} \leq 0$ and therefore according to lemma \ref{lem:locally-optimal-saddle}, $\z^*$ is a locally optimal saddle point. 
\end{proof}   

\subsection{Lemma \ref{lem:sp_are_stable_cegd}}\label{proof:lem:sp_are_stable_cegd}
\begin{lemma_custom_no}{\ref{lem:sp_are_stable_cegd}}
Suppose that assumptions~\ref{ass:smoothness} and~\ref{assum:degeneracy} hold. Let $\z^* := (\pmin^*,\pmax^*)$ be a locally optimal saddle point, i.e.  
\begin{align} 
\nabla f(\z) = 0, \; \nabla^2_\pmin f(\z^*) \succeq \mu_\pmin \I, \;  \nabla^2_\pmax f(\z^*) \preceq - \mu_\pmax \I, \;  (\mu_\pmin, \mu_\pmax >0)
\end{align}
Then iterates of Eq. \eqref{eq:cuvature_gradient_iterates} are stable in $\K_\gamma^*$ as long as 
\begin{align} 
\gamma \leq \min \{ \mu_\pmin/(\sqrt{2} \rho_\pmin),\mu_\pmax/(\sqrt{2} \rho_\pmax) \} 
\end{align} 
\end{lemma_custom_no}
\begin{proof}
The proof is based on a simple idea: in a $\K_\gamma^*$ neighborhood of a locally optimal saddle point, $f$ can not have extreme curvatures, i.e., $\v_\z = \boldsymbol{0}$. Hence, within $\K_\gamma^*$ the update of Eq.~\eqref{eq:cuvature_gradient_iterates} reduces to the gradient update in Eq. \eqref{eq:saddle_gradient}, which is stable according to \cite{nagarajan2017gradient, mescheder2017numerics}.\\
To prove our claim that negative curvature doesn't exist in $\K_\gamma^*$, we make use of the smoothness assumption. Suppose that $\z := (\pmin,\pmax) \in \K_\gamma^*$, then the smoothness assumption~\ref{ass:smoothness} implies 
\begin{align} 
\nabla^2_\pmin f(\z) & =  \nabla^2_\pmin f(\z_*) - \left( \nabla^2_\pmin f(\z_*) - \nabla^2_\pmin f(\z) \right) \\ 
& \succeq \nabla^2_\pmin f(\z_*) - \rho_\pmin \| \z - \z_*\| \I  \\ 
& \succeq \nabla^2_\pmin f(\z_*) - \sqrt{2} \rho_\pmin \gamma \I \\ 
& \succeq (\mu_\pmin - \sqrt{2} \rho_\pmin \gamma) \I  \\ 
& \succ 0  \quad \quad [\gamma < \mu_\pmin/(\sqrt{2} \rho_\pmin)] 
\end{align}
Similarly, one can show that 
\begin{align} 
\nabla^2_\pmax f(\z) \prec 0 \quad \quad  [\gamma < \mu_\pmax/(\sqrt{2} \rho_\pmax)]. 
\end{align}
Therefore, the extreme curvature direction is zero according to the definition in Eq.~\eqref{eq:def:negative_curvature}.
\end{proof}

\subsection{Lemma ~\ref{lemma:escaping}}
\begin{lemma_custom_no} {\ref{lemma:escaping}}
Suppose that $\z^*:= (\pmin^*,\pmax^*)$ is an undesired stationary point of the gradient dynamics, namely
\begin{align} 
\nabla f(\z^*) = 0, \| \v_{\z^*} \| > 0. 
\end{align}
Consider the iterates of Eq.~\eqref{eq:cuvature_gradient_iterates} starting from $\z_0 = (\pmin_0,\pmax_0)$ in a $\gamma$-neighbourhood of $\z^*$. After one step the iterates escape the $\gamma$-neighbourhood of $\z^*$, i.e.
 \begin{align} 
  \| \z_1 - \z^* \| \geq \gamma
 \end{align} 
 for a sufficiently small $\gamma = \bigo( \| \v_{\z^*}\| )$. 
\end{lemma_custom_no}
\begin{proof} 
 \textbf{Preliminaries:} Consider compact notations 
 \begin{align} 
 \nabla_0 := (-\nabla_{\pmin} f(\z_0), \nabla_{\pmax} f(\z_0)), \v_0 := \v_{\z_0}, \v_* = \v_{\z^*} \\ 
 \lambda^{(-)} := \lambda_{\min} \left( \nabla^2_\pmin f (\z^*)\right) < 0 , \lambda^{(+)} := \lambda_{\max} \left( \nabla^2_\pmax f (\z^*)\right) > 0 \\ 
  \lambda^{(-)}_0 := \lambda_{\min} \left( \nabla^2_\pmin f (\z_0)\right) < 0, \lambda^{(+)}_0 := \lambda_{\max} \left( \nabla^2_\pmax f (\z_0)\right) > 0
 \end{align}
 \textbf{Characterizing extreme curvature: } The choice of $\v_0$ ensures that 
 \begin{align} \label{eq:curvature_gradient_product}
 \nabla_0^\top \v_0 > 0 
 \end{align} 
 holds. Since $\z_0$ lies in a $\gamma$-neighbourhood of $\z^*$, we can use the smoothness of $f$ to relate the negative curvature at $\z_0$ to negative curvature in $\z^*$: 
 \begin{align} 
 \nabla^2_\pmin f(\z_0) & \preceq \nabla^2_\pmin f(\z^*) +  \rho_{\pmin}  \| \z_0 - \z^* \| \I  \\ 
& \preceq \nabla^2_\pmin f(\z^*) + \sqrt{2} \rho_\pmin  \gamma \I.
 \end{align}
 Therefore 
 \begin{align} 
 \lambda^{(-)}_0 \leq \lambda^{(-)} + \sqrt{2} \rho_\pmin \gamma  
 \end{align} 
 Similarly, one can show that 
 \begin{align}
     \lambda_0^{(+)} \geq \lambda^{(+)} - \sqrt{2} \rho_\pmax \gamma 
 \end{align}
 Combining these two bounds yields 
 \begin{align} 
 \| \v_0 \| &  = \sqrt{(\lambda^{(-)}_0/(2\rho_\pmin))^2 + (\lambda^{(+)}_0/(2\rho_\pmax))^2} \\ 
 & \geq \frac{1}{4} | \lambda^{(-)}_0/\rho_\pmin | + \frac{1}{4} | \lambda^{(+)}_0/\rho_\pmax |  \\ 
 & \geq \frac{1}{4} \left( |\lambda^{(-)}/\rho_\pmin| + \lambda^{(+)}/\rho_\pmax - 2\sqrt{2} \gamma \right)
 \end{align} 
 To simplify the above bound, we use the compact notation $\lambda := 0.25(|\lambda^{(-)}/\rho_\pmin| + \lambda^{(+)}/\rho_\pmax)$: 
 \begin{align} \label{eq:curvature_lowerbound}
 \| \v_0 \| \geq \lambda - \frac{\sqrt{2}}{2} \gamma
 \end{align} 
\paragraph{Proof of escaping:} The squared norm of the update can be computed as 
 \begin{align} 
 \| \z_{1} - \z^* \|^2  & = \| \z_0 - \z^* + \eta \nabla_0 + \v_{0} \|^2  \\ 
 & = \| \z_0 - \z^* \|^2 + \| \eta \nabla_0 + \v_{0}  \|^2 + 2 (\z_0 - \z^*)^\top (\eta \nabla_0 + \v_0)  \\ 
 & \geq \| \eta \nabla_0 + \v_0 \| \left( \| \eta \nabla_0 + \v_0 \| - 2\sqrt{2} \gamma \right) \label{eq:lower-bound-distance-1}
 \end{align} 
 Now, we plug the results obtained from the smoothness assumption in the above inequality. First, we provide a lower-bound on the sum of gradients and the extreme curvature: 
 \begin{align} 
\| \eta \nabla_0 + \v_0 \| & = \left( \eta^2 \| \nabla_0 \|^2 + \| \v_0 \|^2 + 2 \v_0^\top \nabla_0 \right)^{1/2} \\ 
& \stackrel{\eqref{eq:curvature_gradient_product}}{\geq} \| \v_0 \| \\ 
& \stackrel{\eqref{eq:curvature_lowerbound}}{\geq} \lambda - \frac{\sqrt{2}}{2} \gamma
 \end{align}
Under the condition $\| \eta \nabla_0 + \v_0 \| - 2 \sqrt{2} \gamma>0$, we can use the above lower-bound for inequality~\eqref{eq:lower-bound-distance-1} yielding 
 \begin{align} \label{eq:distance-lower-bound-2}
   \| \z_{1} - \z^* \|^2 \geq  (\lambda - \frac{\sqrt{2}}{2} \gamma) (\lambda - \frac{5}{\sqrt{2}} \gamma)
 \end{align}
\textbf{Choice of $\boldsymbol{\gamma}$:}
To complete our inductive argument, we need to choose $\gamma$ such that the derived lower-bound of Eq.~\eqref{eq:distance-lower-bound-2} is greater than $\gamma^2$, i.e. 
\begin{align} 
& (\lambda - \frac{\sqrt{2}}{2} \gamma) (\lambda - \frac{5}{\sqrt{2}} \gamma) \geq \gamma^2 
\end{align} 
which holds for 
\begin{align} 
\gamma \leq \lambda (\sqrt{2} - \frac{2}{\sqrt{3}}) = \bigo\left( \| \v_{\z^*} \| \right).
\end{align}
Observing that for this choice for the bound of $\gamma$ it holds that
\begin{align}
    \| \eta \nabla_0 + \v_0 \| - 2 \sqrt{2} \gamma &\geq \lambda - \frac{5}{\sqrt{2}} \gamma \\
    & \geq \lambda - \frac{5}{\sqrt{2}} \lambda (\sqrt{2} - \frac{2}{\sqrt{3}})\\
    & = \lambda \frac{5 \sqrt{6} - 12}{3} > 0
\end{align}
concludes the proof.
\end{proof}

\subsection{Guaranteed decrease/increase}
The gradient update step of Eq.~\eqref{eq:saddle_gradient} has the property that an update with respect to $\pmin$ ($\pmax$) decreases (increases) the function value. The next lemma proves that \method{} shares the same desirable property in regions where extreme curvature exists. Note that in regions without extreme curvature, the \method{} method reduces to gradient based optimization and therefore inherits its theoretical properties.
\begin{lemma}\label{lemm:guaranteed_decrease}
In each iteration of  Eq.~\eqref{eq:cuvature_gradient_iterates}, $f$ decreases in $\pmin$ with 
\begin{align} 
f(\pmin_{t+1},\pmax_t) \leq f(\pmin_t,\pmax_t) - (\eta/2) \| \nabla_\pmin f(\z_t)\|^2 +  \lambda_{\pmin}^3/(24\rho_\pmin^2),
\end{align} 
and increases in $\pmax$ with 
\begin{align} 
f(\pmin_{t},\pmax_{t+1}) \geq f(\pmin_t,\pmax_t) + (\eta/2) \| \nabla_\pmax f(\z_t)\|^2 +  \lambda_{\pmax}^3/(24\rho_\pmax^2).
\end{align}
as long as the step size is chosen as 
\begin{align}
    \eta \leq \min \left\{\frac{\sqrt{9 L_\pmin^2\vphantom{L_\pmax^2} + 48 \rho_\pmin \ell_\pmin} - 3 L_\pmin}{8 \rho_\pmin \ell_\pmin}, \frac{\sqrt{9 L_\pmax^2  + 48 \rho_\pmax \ell_\pmax} - 3 L_\pmax}{8 \rho_\pmax \ell_\pmax} \right\} 
\end{align}
\end{lemma} 

\begin{proof}
    The Lipschitzness of the Hessian (Assumption \ref{assum:Lipschitzness}) implies that for $\Delta \in \R^d$
    \begin{align}\label{eq:lipschitz_hessian_bound}
        | f(\pmin+\Delta, \pmax) - f(\pmin, \pmax) - \Delta^\top \nabla_{\pmin}f(\pmin, \pmax) - \frac{1}{2} \Delta^\top \nabla_{\pmin}^2f(\pmin, \pmax) \Delta | \leq \frac{\rho_x}{6} \| \Delta \|^3
    \end{align}
    holds. The update in Eq.~\eqref{eq:cuvature_gradient_iterates} for $\pmin$ is given by $\Delta = \alpha \v - \eta \nabla_{\pmin} f(\pmin, \pmax)$, where $\alpha = -\lambda/(2\rho_\pmin)$ (where we assume w.l.o.g. that $\v^\top \nabla_{\pmin} f(\pmin, \pmax) < 0$) and $\v$ is the eigenvector associated with the minimum eigenvalue $\lambda$ of the Hessian matrix $\nabla_\pmin^2 f(\pmin_t,\pmax_t)$. In the following we use the shorter notation: $\nabla f := \nabla_{\pmin} f(\pmin, \pmax)$, and $\H := \nabla_\pmin^2 f(\pmin,\pmax)$. We can construct a lower bound on the left hand side of Eq. \eqref{eq:lipschitz_hessian_bound} as
    \begin{align}
        | &f(\pmin+\Delta, \pmax) - f(\pmin, \pmax) - \Delta^\top \nabla f - \frac{1}{2} \Delta^\top \H \Delta | \\
        &\geq  f(\pmin+\Delta, \pmax) - f(\pmin, \pmax) - \Delta^\top \nabla f - \frac{1}{2} \Delta^\top \H \Delta \\
        &\geq f(\pmin+\Delta, \pmax) - f(\pmin, \pmax) - \alpha \v^\top \nabla f + (\eta - \frac{\eta^2}{2}L_x) \| \nabla f \|^2 - \frac{1}{2} \alpha^2 \lambda + \alpha \eta \lambda \v^\top \nabla f
    \end{align}
    which leads to the following inequality
    \begin{align}\label{eq:function_decrease_in_pmin_bound}
        f(\pmin+\Delta, \pmax) - f(\pmin, \pmax) &\leq \alpha \v^\top \nabla f - (\eta - \frac{\eta^2}{2}L_x) \| \nabla f \|^2 + \frac{1}{2} \alpha^2 \lambda - \alpha \eta \lambda \v^\top \nabla f + \frac{\rho_x}{6} \| \Delta \|^3 \\
        & \leq \frac{1}{2} \alpha^2 \lambda - (\eta - \frac{\eta^2}{2} L_x) \| \nabla f \|^2 + \frac{\rho_x}{6} \| \Delta \|^3
    \end{align}
    
    By using the triangular inequality we obtain the following bound on the cubic term
    \begin{align} 
    \| \Delta \|^3 & \leq (\eta \| \nabla f \| + \alpha \| \v\|)^3 \\ 
    & \leq 4 \eta^3 \| \nabla f \|^3 + 4 \alpha^3 \| \v\|^3 = 4 \eta^3 \| \nabla f \|^3 + 4 \alpha^3.
    \end{align}
    Replacing this bound into the upper bound of Eq.~\eqref{eq:function_decrease_in_pmin_bound} yields 
    \begin{align}
         f(\pmin+\Delta, \pmax) - f(\pmin, \pmax)  \leq \frac{1}{2} \alpha^2 \lambda - (\eta - \frac{\eta^2}{2} L_x) \| \nabla f \|^2 + \frac{\rho_x}{6} \left( 4 \eta^3 \| \nabla f \|^3 + 4 \alpha^3  \right) 
    \end{align}
    The choice of $\alpha = -\lambda/(2\rho_\pmin)$ leads to further simplification of the above bound: 
    \begin{align} 
      f(\pmin+\Delta, \pmax) - f(\pmin, \pmax)  \leq \frac{\lambda^3}{24 \rho_{\pmin}^2} - \eta (1 - \frac{\eta}{2}L_x - \frac{2}{3} \rho_\pmin \eta^2 \ell_\x) \| \nabla f \|^2 
    \end{align} 
    Now, we choose step size $\eta$ such that 
   \begin{align} 
   1- \frac{\eta}{2}L_x - \frac{2}{3} \rho_\pmin \eta^2 \ell_\pmin \geq 1/2
   \end{align}
    For $\eta \leq \frac{\sqrt{9 L_\pmin^2 + 48 \rho_\pmin \ell_\pmin} - 3 L_\pmin}{8 \rho_\pmin \ell_\pmin}$, the above inequality holds. Therefore, the following decrease in the function value is guaranteed  
   \begin{align} 
    f(\pmin+\Delta, \pmax) - f(\pmin, \pmax)  \leq \lambda^3/(24\rho_\pmin^2) - (\eta/2) \| \nabla f \|^2
   \end{align} 
   Similarly, one can derive the lower-bound for the function increase in $\pmax$. 
\end{proof}
Within a region of extreme curvature, this lemma guarantees a larger decrease (increase) in $\x$ (in  $\y$) compared to gradient descent (ascent). However, we are not claiming that these decrements accelerate the global convergence. 

\subsection{Lemma \ref{lem:linear_transformed_guarantee}}
\begin{lemma_custom_no}{\ref{lem:linear_transformed_guarantee}}
    The set of locally optimal saddle points as defined in Def. \ref{def:local_saddle_point} and the set of stable points of the linear-transformed CESP update method in Eq. \eqref{eq:curvatuer_linear_transformated_gradient} are the same.
\end{lemma_custom_no}
\begin{proof}
    As a direct consequence of lemma \ref{lem:stationary_points_are_optimal} and the positive definiteness property of the linear transformation matrix follows that a locally optimal saddle point is a stationary point of the linear-transformed updates.
    
    In the following, we prove stability of locally optimal saddles. Let's consider a locally optimal saddle point $\z^* := (\pmin^*,\pmax^*)$ in the sense of Def. \ref{def:local_saddle_point}. From lemma \ref{lem:locally-optimal-saddle} follows that
\begin{align} \label{equ:pd_nd_assumption}
\nabla^2_{\pmin} f(\pmin^*,\pmax^*) \succeq \mu_\pmin \I ,\qquad \nabla^2_{\pmax} f(\pmin^*,\pmax^*) \preceq  - \mu_\pmax \I.
\end{align} 
for $\mu_\pmin, \mu_\pmax > 0$. As a direct consequence the extreme curvature direction is zero, i.e. $\v_\z = \boldsymbol{0}$. Hence, the Jacobian of the update in Eq.~\eqref{eq:curvatuer_linear_transformated_gradient} is given by
\begin{align}
    \I + \eta \A_{\z^*} \H(\z^*).
\end{align}
where $\H(\z^*)$ is the partial derivative of $(-\nabla_\pmin f(\z^*), \nabla_\pmax f(\z^*))$. The point $\z^*$ is a stable point of the dynamic if all eigenvalues of its Jacobian lie within the unit disk. This condition can be fulfilled, with a sufficiently small step size, if and only if all the real parts of the eigenvalues of $\A_{\z^*} \H(\z^*)$ are negative. \\
Hence, to prove stability of the update for a locally optimal saddle point $\z^*$, we have to show that the following expression is a Hurwitz matrix~\cite{khalil2002nonlinear}:
\begin{align}
     \J(\z^*) = \underbrace{\A_{\z}}_{:=\A} \underbrace{\mymatrix{-\nabla_{\pmin}^2f(\z^*)}{-\nabla_{\pmin,\pmax}f(\z^*)}{\nabla_{\pmax,\pmin}f(\z^*)}{\nabla^2_{\pmax} f(\z^*)}}_{=\H} := \J
\end{align}
Since $\A$ is a symmetric, positive definite matrix, we can construct its square root $\A^\frac{1}{2}$ such that $\A = \A^\frac{1}{2}\A^\frac{1}{2}$. The matrix product $\A \H$ can be re-written as
\begin{align}
    \A \H = \A^\frac{1}{2} (\A^\frac{1}{2} \H \A^\frac{1}{2}) \A^{-\frac{1}{2}}.
\end{align}
Since we are multiplying the matrix $\Tilde{\J} = \A^\frac{1}{2} \H \A^\frac{1}{2}$ from the left with the inverse of the matrix from which we are multiplying from the right side, we can observe that $\J$ has the same eigenvalues as $\Tilde{\J}$.
The symmetric part of $\Tilde{\J}(\z^*)$ is given by
\begin{align}
    \frac{1}{2}\left(\Tilde{\J} + \Tilde{\J}^\top\right) &= \frac{1}{2}(\A^\frac{1}{2} \H \A^\frac{1}{2} + \A^\frac{1}{2} \H^\top \A^\frac{1}{2}) = \A^\frac{1}{2} (\H + \H^\top) \A^\frac{1}{2} \\
    &= \A^\frac{1}{2} \mymatrix{-\nabla_{\pmin}^2f(\z^*)}{0}{0}{\nabla^2_{\pmax} f(\z^*)}\A^\frac{1}{2}
\end{align}

From the assumption in Eq. \eqref{equ:pd_nd_assumption} follows that the block diagonal matrix $(\H + \H^\top)$ is a symmetric, negative definite matrix, for which it therefore holds that $x^\top(\Tilde{\J} + \Tilde{\J}^\top)x \leq 0$ for any $x \in \mathrm{R}^{k+d}$. The remaining part of the proof follows the argument from \cite{Benzi05numericalsolution} Theorem 3.6.\\

Let ($\lambda, v$) be an eigenpair of $\Tilde{\J}$. Then, the following two equalities hold:
\begin{align}
    v^* \Tilde{\J} v &= \lambda \\
    (v^* \Tilde{\J} v)^* &= v^* \Tilde{\J}^\top v = \bar{\lambda}
\end{align}
Therefore, we can re-write the real part of the eigenvalue $\lambda$ as:
\begin{align}
    \text{Re}(\lambda) = \frac{\lambda+\bar{\lambda}}{2} = \frac{1}{2} v^* (\Tilde{\J} + \Tilde{\J}^\top) v.
\end{align}
By observing that 
\begin{align}
    v^* (\Tilde{\J} + \Tilde{\J}^\top) v = \text{Re}(v)^\top (\Tilde{\J} + \Tilde{\J}^\top) \text{Re}(v) + \text{Im}(v)^\top (\Tilde{\J} + \Tilde{\J}^\top) \text{Im}(v)
\end{align}
is a real, negative quantity, we can be sure that the real part of any eigenvalue of $J$ is negative. Therefore it directly follows that, with a sufficiently small step size $\eta > 0$, any locally optimal saddle point $\z^*$ is a stable stationary point of the linear-transformed update method in Eq.~\eqref{eq:curvatuer_linear_transformated_gradient}.
\end{proof}

\section{Transformed Gradient Updates}\label{sec:transformed_updates}
Table \ref{tab:upd-matrices} shows the update matrices for commonly used optimization methods on the saddle point problem.
\begin{table}[h]
\centering
\caption{Update matrices of the different optimization schemes.}
\label{tab:upd-matrices}
\begin{tabular}{|l | c|c|} \toprule
                                                      & Formula & positive definite? \\ \midrule
\multirow{2}{*}{Gradient Descent}                     & $\mathcal{A}_{t} =  I$         &   \multirow{2}{*}{Yes.} \\
                                                      & $\mathcal{B}_{t} =  I$        &                    \\ \hline
\multirow{2}{*}{Adagrad \cite{Duchi:EECS-2010-24}}                            & $\mathcal{A}_{t, ii} = \left(\sqrt{\sum_{\tau = 1}^t \left( \nabla_{\pmin_{i}} f(\pmin_{\tau}, \pmax_{\tau})\right)^2 + \epsilon}\right)^{-1}$   & \multirow{2}{*}{Yes.}\\ 
                                                      &  $\mathcal{B}_{t, ii} =  \left(\sqrt{\sum_{\tau = 1}^t \left( \nabla_{\pmax_{i}} f(\pmin_{\tau}, \pmax_{\tau})\right)^2 + \epsilon}\right)^{-1}$       &                 \\\hline   
\multirow{2}{*}{Saddle-Free Newton \cite{NIPS2014_5486}} &  $\mathcal{A}_t =  \left| \nabla^2_{\pmin} f(\pmin_{t}, \pmax_{t}) \right|^{-1}$       &  \multirow{2}{*}{Yes.}\\
                                  &   $\mathcal{B}_t =  \left| \nabla^2_{\pmax}  f(\pmin_{t}, \pmax_{t}) \right|^{-1}$       &                      \\ \bottomrule
\end{tabular}
\end{table}

\section{Experiments} \label{sect:experiments_appendix}

\subsection{Toy Example}
Figure \ref{fig:convconv_example_basin_attraction} shows the basin of attraction for {\sc GD} and \method{} on the toy saddle point problem of Eq. \eqref{eq:conv_conv_function}.
\begin{figure*}[h]
    \begin{subfigure}[t]{0.48\textwidth}
        \centering
        \includegraphics[width=1\textwidth]{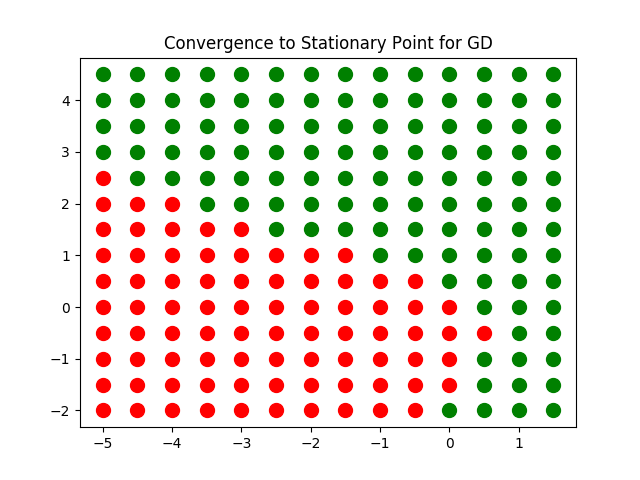}
        \caption{}
    \end{subfigure}%
    ~ 
    \begin{subfigure}[t]{.48\textwidth}
        \centering
        \includegraphics[width=1\textwidth]{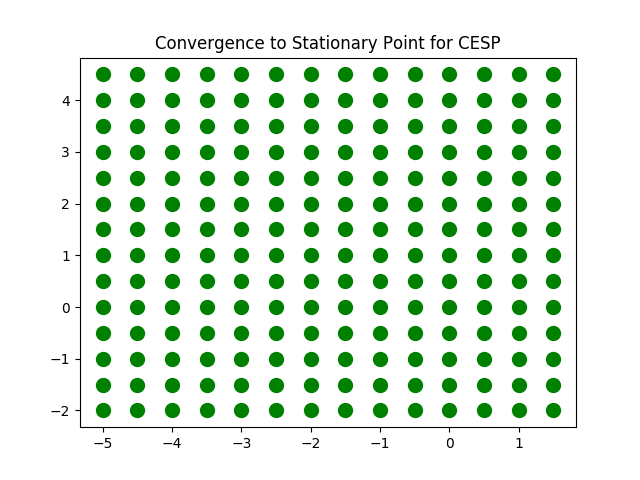}
        \caption{}
    \end{subfigure}
    \caption{Comparison of the basin of attraction for GD and \method{} to the locally optimal saddle point (green area) and the undesired critical point (red area).
    }
    \label{fig:convconv_example_basin_attraction}
\end{figure*}

\subsection{Generative Adversarial Networks} \label{sec:gan_experiments_appendix}
\subsubsection{Single-layer GAN}

\begin{table}[h]
\centering
\caption{Parameters of the single-layer GAN model.}
\label{tab:model_params}
\begin{tabular}{l|r|r}
                     & Discriminator & Generator \\ \hline \hline
Input Dimension      & 784                               & 10                            \\\hline
Hidden Layers        & 1                                 & 1                             \\\hline
Hidden Units / Layer & 100                                & 100                            \\\hline
Activation Function  & Leaky ReLU                              & Leaky ReLU                          \\\hline
Output Dimension     & 1                                 & 784                           \\\hline
Batch Size           & \multicolumn{2}{c}{1000}                                          \\\hline
Learning Rate $\eta$        & \multicolumn{2}{c}{0.01}\\\hline
Learning Rate $\alpha := \frac{1}{2\rho_{\pmin}} = \frac{1}{2\rho_{\pmax}}$        & \multicolumn{2}{c}{0.05} 
\\
\end{tabular}
\end{table}

\paragraph{Using two individual loss functions}
It is common practice in GAN training to not consider the saddle point problem as defined in Eq. \eqref{eq:GAN_objective}, but rather split the training into two individual optimization problems over different functions. In particular, one usually considers
\begin{align}
    \min_{\pmin} ( f_1(\pmin, \pmax) &= - \mathbb{E}_{\z \sim p_z} \log D_\pmin (G_\pmax (\z)) ) \\
    \max_\pmax (f_2(\pmin, \pmax) &= \mathbb{E}_{\theta \sim p_d} \log D_{\pmin}(\theta) + \mathbb{E}_{\mathbf{z} \sim p_z} \log(1 - D_{\pmin}(G_{\pmax}(\mathbf{z}))) )
\end{align}
Our \method{} optimization method is defined individually for the two parameter sets $\pmin$ and $\pmax$ and can therefore also be applied on such a setting with two individual objectives. Figure \ref{fig:GAN_exp_two_loss} shows the results on the single-layer GAN problem, trained with two individual losses. In this experiment, \method{} decreases the negative curvature of $\nabla^2_\pmin f_1$, while the gradient method can not exploit the negative curvature appropriately (the value of the smallest eigenvalue is oscillating in the negative area).

\begin{figure*}
    \centering
    \begin{subfigure}[t]{0.33\textwidth}
        \centering
        \includegraphics[width=1\textwidth]{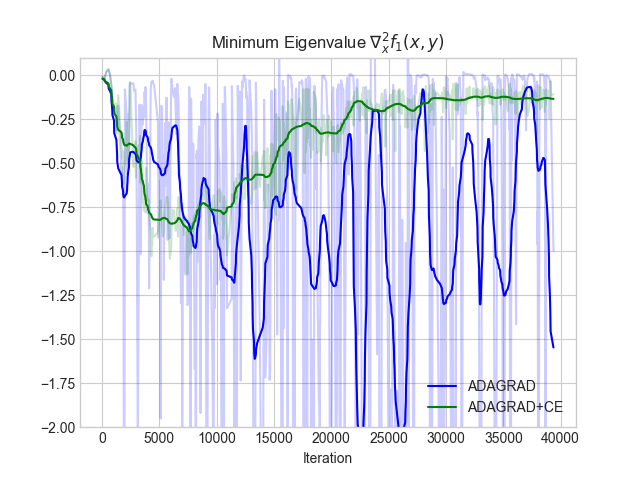}
    \end{subfigure}%
    ~ 
    \begin{subfigure}[t]{.33\textwidth}
        \centering
        \includegraphics[width=1\textwidth]{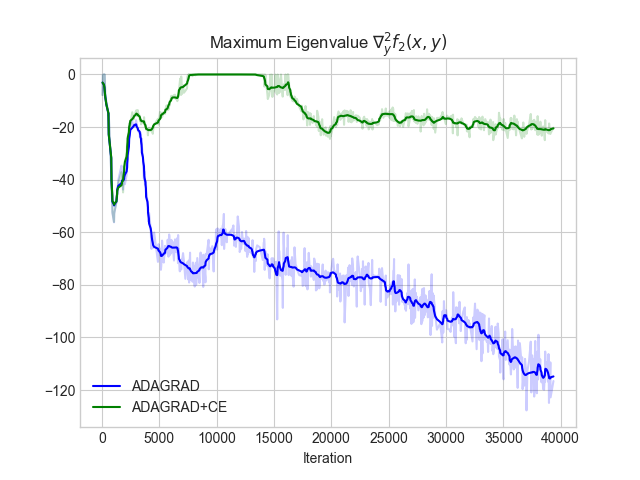}
    \end{subfigure}%
    ~ 
    \begin{subfigure}[t]{.33\textwidth}
        \centering
        \includegraphics[width=1\textwidth]{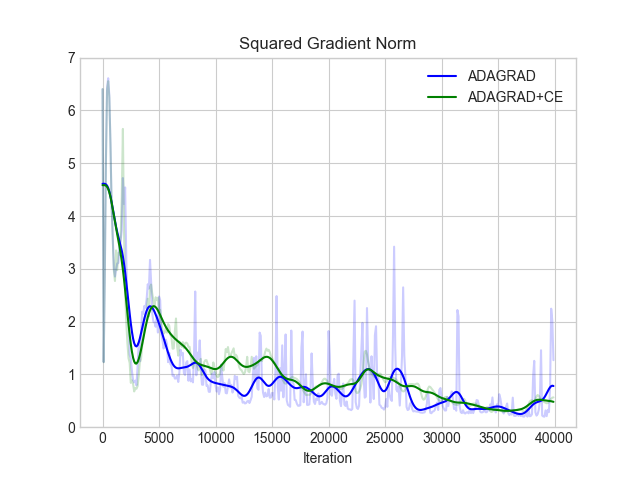}
    \end{subfigure}
    \caption{Results of the single-layer GAN with individual loss functions on MNIST data. The first two plots show the minimum eigenvalue of $\nabla_{\pmin}^2 f_1(\pmin, \pmax)$ and the maximum eigenvalue of $\nabla_{\pmax}^2 f_2(\pmin, \pmax)$, respectively. The third plot shows $\| \nabla f(\z_t)\|^2$. The transparent graph shows the original values, whereas the solid graph is smoothed with a Gaussian filter.}
    \label{fig:GAN_exp_two_loss}
\end{figure*}

\end{document}